%% file: main.tex
\documentclass[10pt,twocolumn,letterpaper]{article}

\usepackage[pagenumbers]{cvpr} %

\input{preamble}

\definecolor{cvprblue}{rgb}{0.21,0.49,0.74}
\usepackage[pagebackref,breaklinks,colorlinks,allcolors=cvprblue]{hyperref}

\def\paperID{} %
\def\confName{CVPR}
\def\confYear{2026}

\title{Planning in 8 Tokens: A Compact Discrete Tokenizer for Latent World Model}

\author{
Dongwon Kim$^1$\quad
Gawon Seo$^2$\quad
Jinsung Lee$^2$\quad
Minsu Cho$^{2,3}$\quad
Suha Kwak$^2$\\[4pt]
$^1$KAIST\quad
$^2$POSTECH\quad
$^3$RLWRLD\\[2pt]
{\tt\small \url{https://kdwonn.github.io/CompACT}
}}

\begin{document}
\maketitle
\input{sec/0_abstract}    
\input{sec/1_intro}
\input{sec/2_related_work}

\input{sec/3_method}
\input{sec/4_experiment}
\input{sec/5_conclusion}

\clearpage
{
    \small
    \bibliographystyle{ieeenat_fullname}
    \bibliography{main}
}
\clearpage

\input{sec/X_suppl}

\end{document}

%% file: preamble.tex
\usepackage{xcolor}
\usepackage{soul}

\newcommand{\modelname}{CompACT\xspace}

\usepackage{duckuments}
\usepackage{url}            %
\usepackage{booktabs}       %
\usepackage{amsfonts}       %
\usepackage{nicefrac}       %
\usepackage{microtype}      %
\usepackage{xcolor}         %
\usepackage{bbding}
\usepackage{multirow}
\usepackage{algorithm}
\usepackage{algorithmic}
\usepackage{graphicx, amsmath, amssymb, caption, subcaption, overpic, textpos}
\usepackage{arydshln} %
\usepackage{listings}
\usepackage{setspace}
\usepackage{xspace}
\usepackage[table]{xcolor}

\usepackage{amsthm}
\newtheorem{definition}{Definition}
\newtheorem{proposition}{Proposition}

\input{math_command}

%% file: math_command.tex
\usepackage{amsmath,amsfonts,bm}

\def\eqref#1{equation~\ref{#1}}

\def\1{\bm{1}}

\def\va{{\bm{a}}}

\def\vo{{\bm{o}}}

\def\vz{{\bm{z}}}

\def\mA{{\bm{A}}}

\def\mO{{\bm{O}}}

\DeclareMathAlphabet{\mathsfit}{\encodingdefault}{\sfdefault}{m}{sl}
\SetMathAlphabet{\mathsfit}{bold}{\encodingdefault}{\sfdefault}{bx}{n}

\newcommand{\R}{\mathbb{R}}

\DeclareMathOperator*{\argmin}{arg\,min}

%% file: sec/0_abstract.tex
\begin{abstract}
World models provide a powerful framework for simulating environment dynamics conditioned on actions or instructions, enabling downstream tasks such as action planning or policy learning.
Recent approaches leverage world models as learned simulators, but its application to decision-time planning remains computationally prohibitive for real-time control.
A key bottleneck lies in latent representations: conventional tokenizers encode each observation into hundreds of tokens, making planning both slow and resource-intensive.
To address this, we propose \modelname, a discrete tokenizer that compresses each observation into as few as 8 tokens, drastically reducing computational cost while preserving essential information for planning.
An action-conditioned world model that occupies \modelname tokenizer achieves competitive planning performance with orders-of-magnitude faster planning, offering a practical step toward real-world deployment of world models.
\end{abstract}

%% file: sec/1_intro.tex
\section{Introduction}
\label{sec:intro}
Humans navigate the world not through pixel-perfect recall of their surroundings, but rather through compact mental representations that capture only the information necessary for decision-making~\cite{forrester1971counterintuitive,ha2018world}.
This internal model—an imprecise but efficient abstraction of reality—reduces the complexity of sensory input into a representation optimized for action and planning.
In the context of artificial intelligence and reinforcement learning (RL), this concept manifests as the \textit{world model}~\cite{ha2018world}, a neural network that captures environment dynamics to enable planning~\cite{micheli2022transformers,hansen2023td, zhou2024dino, bar2025navigation} and policy learning~\cite{hafner2020mastering, alonso2024diffusion, hafner2023mastering, hafner2019dream, yang2023learning}.

World models have emerged as a promising solution to the sample inefficiency of RL. Traditional model-free RL methods require millions of interactions with the environment to learn effective policies, making them impractical for real-world applications where data collection is expensive or risky. By learning to predict future states, world models enable agents to simulate experiences internally, reducing the need for real environment interactions. Furthermore, these models themselves can be used for planning without additional learning of policy~\cite{zhou2024dino,bar2025navigation} through model-predictive control (MPC)~\cite{williams2016aggressive,de2005cem}.

Recent advances in world modeling have been driven by the rapid progress in generative models, particularly in image and video generation~\cite{rombach2022high, esser2021taming, chang2022maskgit}.
These models can generate photorealistic images or videos conditioned on language instructions~\cite{yang2023learning,du2023video} or actions~\cite{bar2025navigation,zhou2024dino,alonso2024diffusion,yang2023learning,zhu2024irasim}, suggesting an implicit understanding of world's underlying dynamics.

{However, there exists a critical gap between these generative approaches and their application to planning. These models are designed for photorealistic image generation, requiring them to capture extensive perceptual detail such as textures, lighting, and shadows.}
This necessitates encoding single images into hundreds of latent tokens, which sharply increases computational cost.
Since most world models in the literature adopt attention-based architectures~\cite{peebles2023scalable}, this burden grows quadratically, making planning especially expensive.
As a result, current world models remain impractical for real-world control:
for example, 
the state-of-the-art 
navigation world models (NWM)~\cite{bar2025navigation} require up to 3 minutes of computation per episode for planning,\footnote{Measured using a single RTX 6000 ADA GPU.} making them unsuitable for applications demanding real-time responsiveness.
{This motivates us to explore an alternative design philosophy: \textit{what if we prioritize extreme compression over perfect reconstruction?} Rather than seeking the conventional path of increasing token count for higher fidelity representations, we hypothesize that aggressive compression might actually be beneficial—forcing the world model to learn more abstract, action-relevant representations rather than preserving every perceptual detail. To test this hypothesis, we push compression to its extreme limit and investigate whether such radical reduction can still support effective planning.
}

We propose \modelname , a compact tokenizer that encodes each image as few as 8 tokens—approximately 128 bits per image (8 tokens of 16 bits each). 
This represents an extreme compression ratio compared to existing approaches. For instance, the SD-VAE tokenizer~\cite{rombach2022high} used in NWM~\cite{bar2025navigation} requires 784 tokens to represent the same image. 
Beyond the reduction in token count, our tokenizer further distinguishes itself by employing a discrete latent space, 
enabling much faster future-state prediction: each token is unmasked only once~\cite{chang2022maskgit}, rather than being processed through hundreds of iterative denoising steps typically required in diffusion models utilizing continuous latent space~\cite{ho2020denoising}.
By training world models in this compact latent space, we can achieve order-of-magnitude reductions in rollout latency.

Compressing each image to just 128 bits  creates an irreducible information bottleneck—the question is not whether to lose information, but which information to preserve. 
Planning requires low-frequency features such as high-level semantics and spatial relationships, rather than high-frequency perceptual details like textures and lighting. 
Our approach separates 
these two aspects: only planning-critical semantics are \textit{preserved} in compact tokens, while perceptual details are \textit{synthesized} when pixel-level outputs are needed during decoding. 

The key design choice enabling such selective preservation of semantic information is our use of a \textit{frozen} pretrained vision encoder~\cite{simeoni2025dinov3} as the foundation of our tokenizer. 
Conventional tokenizers train encoders end-to-end for reconstruction, prioritizing perceptual fidelity. 
In contrast, we leverage the rich semantic representations already captured by vision foundation models. 
Our compact latent tokens act as learnable queries that attend to these frozen representations via a cross-attention-based resampling module. 
Crucially, because vision foundation models already abstract away low-level reconstruction details—focusing instead on semantic understanding—our resampling process can only distill planning-critical semantic information. 
This design inherently ensures the tokenizer preserves object-level semantics and spatial relationships over photorealistic details.

Complementing this semantic encoding strategy, our second key contribution is a generative decoding approach: rather than attempting direct pixel reconstruction from 16 or 8 tokens, our decoder learns to unmask a latent representation capturing perceptual details from a pretrained target tokenizer that uses hundreds of tokens per image (specifically, the VQGAN tokenizer from MaskGIT~\cite{chang2022maskgit}), using our compact tokens as conditioning. 
While our compact latent tokens capture only high-level semantic features, the generative decoding process synthesizes fine-grained details that are consistent with these semantics.
This formulation transforms an intractable decompression problem into a tractable conditional generation task. 

To validate the effectiveness of the proposed approach, we train action-conditioned world models on the latent space of \modelname for both navigation and robot manipulation tasks.
Such action-conditioned world models have a unique strength in that they can serve as general-purpose planners via MPC, but the prohibitive computational burden required for rollouts has remained as a bottleneck.
On navigation planning in RECON~\cite{shah2021rapid}, an action-conditioned world model trained with \modelname achieves comparable accuracy to one using 784 continuous tokens while delivering approximately 40× speedup in planning latency. 
Furthermore, our 8-token model outperforms previous tokenizer with 64 tokens, validating that carefully designed extreme compression can yield both computational efficiency and superior planning performance.
To further validate the efficacy of compact latent tokens learned by \modelname, we conduct action-conditioned video prediction experiments on RoboNet~\cite{dasari2019robonet}. 
On RoboNet, \modelname latent tokens enable accurate action regression comparable to previous tokenizers using 16× more tokens, and maintain strong action consistency in generated videos, confirming that the learned representations preserve action-relevant information critical for accurate planning.

\begin{figure*}[h!]
    \centering
    \includegraphics[width=\linewidth]{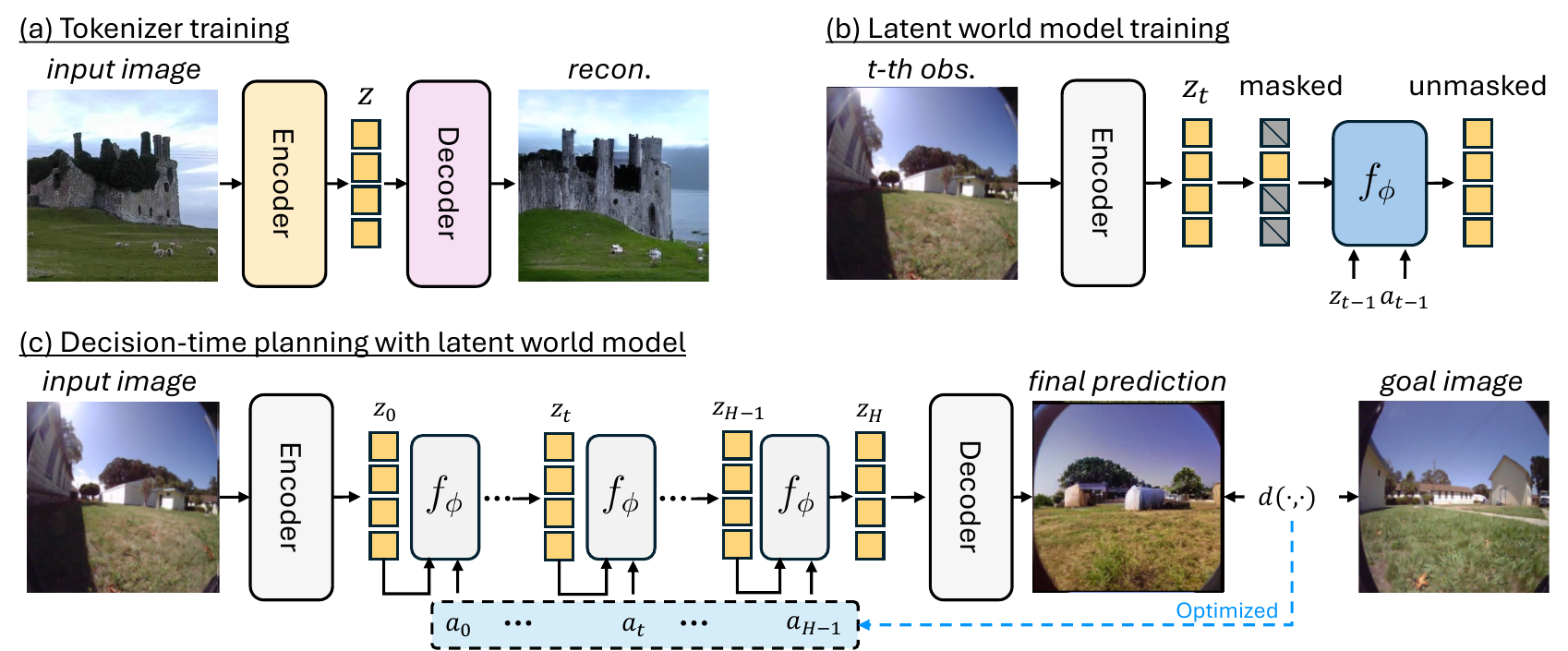}
    \caption{
\textbf{Overview of the proposed latent world model formulation (Sec.~\ref{sec:formulation}).}
(a) An image tokenizer is first trained with a reconstruction objective to map an input image into compact latent tokens $\vz$. (Fig.~\ref{fig:tok_detail} and Sec.~\ref{sec:compacttok}). 
(b) Using the learned tokenizer, latent world model $f_{\phi}(\vz_t, \va_t)$ is trained to model the conditional distribution of the future state $p_{\phi}(\vz_{t+1} | \vz_{t}, \va_{t})$, where we adopt masked generative modeling (Sec.~\ref{sec:compact_lwm}).
(c) At test time, the learned latent world model is used for \textit{decision-time planning}: An optimization procedure (e.g., MPC with CEM) searches over actions $\va_{0:H-1}$ to minimize the distance between the predicted final state and a goal image. 
}
    \label{fig:overall}
\end{figure*}

%% file: sec/2_related_work.tex
\section{Related Work}
\label{sec:related}

\subsection{Image tokenization}
Image tokenization has played a crucial role in visual generation by alleviating the difficulty of directly modeling distributions in high-dimensional space~\cite{van2017neural,esser2021taming,rombach2022high,yu2021vector,lee2022autoregressive,mentzer2023finite}.


Conventional image tokenizers rely on 2D patch-grid latents, whose token count is tied to the input resolution, preventing further compaction.
Representing images with a few learned tokens has long been studied for recognition and vision-language modeling, \eg, TokenLearner~\cite{ryoo2021tokenlearner}, Slot Attention~\cite{locatello2020object}, Perceiver~\cite{jaegle2021perceiver}, and Q-Former~\cite{li2023blip}.
Notably, TokenLearner pools features into as few as 8 tokens via learned spatial attention, showing that a handful of task-adaptive tokens suffices for image and video recognition.
In a similar spirit, recent works have explored 1D tokenization~\cite{yu2024image,bachmann2025flextok,miwa2025one,kim2025democratizing}, which does not explicitly preserve spatial structure; FlexTok~\cite{bachmann2025flextok}, for instance, supports 1--256 tokens with later tokens capturing finer details.
However, these tokenizers target photorealistic generation, prioritizing high-frequency details and faithful reconstruction that are largely irrelevant to decision making and planning.

Recent tokenizers that leverage pretrained vision foundation models as encoders~\cite{zheng2025diffusion,gao2025one}, share architectural similarities with the proposed approach. However, their use of foundation model features is motivated by improving the tractability of downstream generative modeling, rather than achieving extreme compression as in \modelname.

Several recent world models~\cite{hafner2020mastering,hafner2023mastering,micheli2022transformers,
micheli2024efficient,scannell2025discrete,wu2024ivideogpt} relate to the proposed approach through their use of discrete latent representations.
Most related to ours, \cite{micheli2024efficient,wu2024ivideogpt} reduce per-frame token count by conditioning on previous frames, but this limits applicability to long-horizon planning or scenarios with significant viewpoint changes.
\modelname instead achieves compact tokenization unconditionally, by learning to 
retain only planning-critical information.

\subsection{Masked generative model}
Masked image generative models~\cite{chang2022maskgit,chang2023muse,gao2023mdtv2,li2023mage,fan2024fluid,li2024autoregressive,weber2024maskbit} leverage bidirectional attention mechanisms to reconstruct masked tokens during generation. Unlike traditional autoregressive models~\cite{chen2020generative,tang2024hart} predicting tokens one-by-one, these architectures~\cite{yu2023language,zheng2022movq} can sample multiple tokens within a single step, thereby reducing the number of steps needed for full image generation. Notably, MaskGIT~\cite{chang2022maskgit} and MAR~\cite{li2024autoregressive} have demonstrated that such designs enable both rapid and high-quality image synthesis. 
In this work, we focus on the tokenization stage and adopt the widely used non-autoregressive sampling approach from MaskGIT~\cite{chang2022maskgit} for generating token sequences.

\subsection{Planning via World Models}
World models \cite{ha2018world} serve as internal representations that encode environmental dynamics, enabling agents to mentally simulate future states before acting. By predicting future observations from current states and actions, these models facilitate planning across diverse domains including robotics \cite{yang2023learning,mendonca2023structured}, autonomous driving \cite{hu2023gaia,gao2024vista,zhao2025drivedreamer}, gaming \cite{bruce2024genie,alonso2024diffusion,valevski2024diffusion}, and navigation \cite{bar2025navigation,koh2021pathdreamer,nie2025wmnav,yao2025navmorph}. 

Existing approaches can be broadly categorized into two paradigms based on their planning mechanisms. 
One line of approaches employs decision-time planning with world models through test-time optimization, where methods like TDMPC2 \cite{hansen2023td}, DINO-WM~\cite{zhou2024dino}, and NWM \cite{bar2025navigation} iteratively refine action sequences toward specified goals using action-conditioned world model. 
Another line of approaches adopts hierarchical planning through subgoal generation, where sparse intermediate visual states are first generated to bridge current observations to goals, followed by Inverse Dynamics Models to extract executable actions. 
UniPi \cite{du2023learning} exemplifies this strategy through conditional video generation guided by textual goals, and AVDC \cite{ko2023learning} uses language-conditioned prediction with optical flow for action estimation. 

These existing approaches face significant computational challenges in real-time settings, especially with large models like diffusion-based video generation models \cite{du2023learning,hu2023gaia,yang2023learning,bruce2024genie}. 
In this work, we aim to build a world model within the extremely compact latent space, enabling more efficient planning and control. 
To validate the generality of our approach, we evaluate it across both paradigms: goal-conditioned visual navigation, corresponding to decision-time planning with world models, and action-conditioned video prediction, corresponding to hierarchical planning through subgoal generation.

%% file: sec/3_method.tex
\section{Method}

\subsection{Latent generative model as world model}
\label{sec:formulation}
In this section, we first describe how a world model can be formulated as a latent generative model. The overall formulation is depicted in Fig.~\ref{fig:overall}.
We consider the standard world model setting where the objective is to predict future observations given current state and action.
Formally, we denote observations (e.g., video frames) as $\mO = [\vo_0, \vo_1, \dots, \vo_T] \in \mathbb{R}^{T \times H \times W \times 3}$ and actions as $\mA = [\va_0, \va_1, \dots, \va_T] \in \mathbb{R}^{T \times 3}$.\footnote{In navigation settings, actions are 3-dimensional, representing changes in $x$-axis, $y$-axis, and yaw. The formulation generalizes to different action dimensions (e.g., 5-dimensional actions of a robot arm).}
The world model $f_\mathtt{\theta}: \mathbb{R}^{H \times W \times 3} \times \mathbb{R}^{3} \rightarrow \mathcal{P}(\mathbb{R}^{H\times W \times 3})$ can be formulated as:
\begin{align}
\begin{split}
f_{\theta}: (\vo_t, \va_t) &\mapsto p_{\theta}(\vo_{t+1} | \vo_{t}, \va_{t}). 
\end{split}
\label{eq:world_obs}
\end{align}
For simplicity, we omit the temporal context window in our notation; in practice, the model conditions on a history of $\tau$ observations and actions.

Because real-world dynamics are inherently uncertain and only partially observable, a world model should produce a stochastic distribution over future states rather than a deterministic prediction.
Such a stochastic formulation of the world model can be naturally implemented using generative modeling, where the generator is conditioned on past observations 
$\vo_{t}$ and action $\va_t$.
Direct generative modeling in pixel space is computationally prohibitive due to 
the high dimensionality of visual observations.
Instead, {the world model $f_\theta$} can be formulated to operate on low-dimensional latent tokens $\vz \in \mathbb{R}^{N \times D}$~\cite{bar2025navigation}.
These latent tokens are obtained via an image tokenizer comprising an encoder $\mathcal{E}: \mathbb{R}^{H \times W \times 3} \rightarrow \mathbb{R}^{N \times D}$ and decoder $\mathcal{D}: \mathbb{R}^{N \times D} \rightarrow \mathbb{R}^{H \times W \times 3}$, trained with a reconstruction objective: $\mathcal{L}_\textrm{recon} = ||\vo - \mathcal{D}(\mathcal{E}(\vo))||^2_2$ (Fig.~\ref{fig:overall}(a)).
Extending Eq.~(\ref{eq:world_obs}), a latent world model $f_{\phi}: \mathbb{R}^{N \times D} \times \mathbb{R}^{3} \rightarrow \mathcal{P}(\mathbb{R}^{N \times D})$ can be described as 
\begin{align}
\begin{split}
f_{\phi}: (\vz_t, \va_t) \mapsto p_{\phi}&(\vz_{t+1} | \vz_{t}, \va_{t}),
\end{split}
\label{eq:world}
\end{align}
{where $\vz_{t} = \mathcal{E}(\vo_{t})$ .
Here, the token count $N$ directly determines computational complexity: for attention-based architectures~\cite{peebles2023scalable} commonly used in generative models, cost scales quadratically with $N$.
By keeping $N$ small, the latent world model formulation alleviates this quadratic bottleneck and enables efficient decision-time planning.
}

{Once the latent world model $f_\theta$ is trained, we can use it to find a sequence of actions $\{\va_t\}$ that drives the transition from the initial observation $\vo_0$ to the goal observation $\vo_\textrm{goal}$, as illustrated in Fig.~\ref{fig:overall}(c).
We first compute $\vz_0 = \mathcal{E}(\vo_0)$, and initialize a candidate action sequence $\mathbf{a} = [\va_0, \va_{1}, \ldots, \va_{H-1}]$.
Then, we obtain a sequence of latent tokens $\{\vz_t\}$ by rolling out the trained world model to predict future states over the planning horizon $H$:
}
\begin{align}
\vz_{t+1} \sim f_\phi(\vz_t, \va_t), ~t \in \{0, \cdots, H-1\}.
\end{align}
{
After the rollout reaches the planning horizon (i.e., $\vz_H$ is sampled), the candidate action sequence $\mathbf{a}$ is evaluated using a cost function that measures the distance between the final predicted observation and the goal: 
{$C(\mathbf{a}) = d(\hat{\vo}_{H}, {\vo}_\textrm{goal})$}, where $\hat{\vo}_H = \mathcal{D}(\vz_H)$, $\hat{\vo}_\textrm{goal} = \mathcal{D}(\vz_\textrm{goal})$, and $d(\cdot, \cdot)$ is a distance measure (e.g., LPIPS~\cite{johnson2016perceptual}).\footnote{We can also define the cost function as $d(\vz_{H}, \vz_\textrm{goal})$ in the latent space, which enables the faster planning since we can skip the decoding (Tab.~\ref{tab:ctx_masking_cost}).}
The optimal action sequence is then obtained via solving $\mathbf{a}^* = \argmin_{\mathbf{a}} C(\mathbf{a})$, where the optimization can be performed using sampling-based methods~\cite{de2005cem,chua2018cem,williams2016aggressive} or gradient descent.
}

\begin{figure}[t!]
    \centering
    \includegraphics[width=\columnwidth]{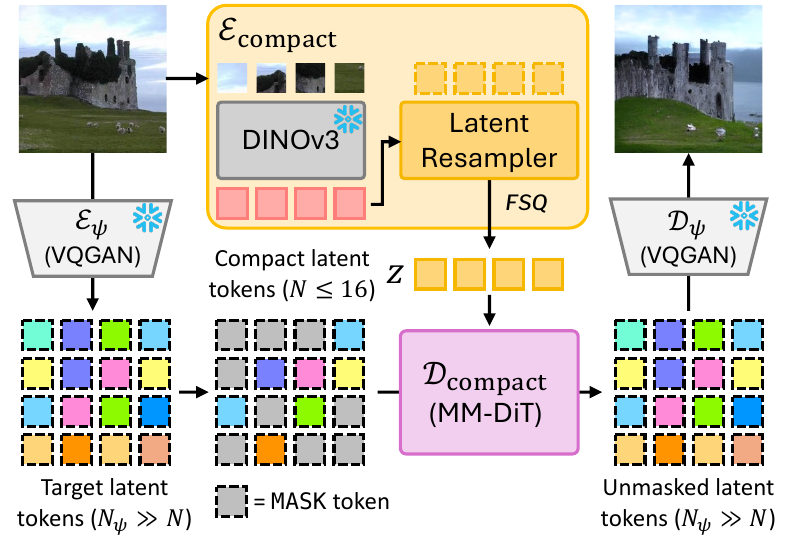}
    \vspace{-7mm}
    \caption{\textbf{A tokenizer architecture detail.} During training, only the latent resampler and $\mathcal{D}_\textrm{compact}$ are updated. $\mathcal{E}_{\psi}$ produces masked target tokens (training only), while $\mathcal{D}_{\psi}$ is used only during inference for pixel level reconstruction.}
    \label{fig:tok_detail}
\end{figure}

\subsection{\modelname tokenizer}
\label{sec:compacttok}
The computation bottleneck in world model planning stems from the latent token count $N$: conventional tokenizers typically encode images with hundreds of tokens, which slows down their sampling during autoregressive rollout. 
We introduce \modelname, a compact tokenizer $\mathcal{D}_\textrm{compact} \circ \mathcal{E}_\textrm{compact}$ that encodes each image into just 16 or 8 discrete tokens and avoids iterative denoising by using a discrete latent space (Fig.~\ref{fig:tok_detail}).
Despite this extreme compression, \modelname still preserves the sufficient information for 
planning (Sec.~\ref{sec:experiment}).

\subsubsection{Semantic encoding via frozen features}
\label{sec:comppacttok_enc}
The key design principle of our tokenizer is to preserve only planning-critical semantic information while discarding reconstruction-oriented high-frequency features. To achieve this, we build our encoder $\mathcal{E}_\textrm{compact}$ around a \textit{frozen} pretrained vision encoder—specifically, DINOv3~\cite{simeoni2025dinov3}—which already abstracts away low-level visual details in favor of semantic understanding. 
The encoder $\mathcal{E}_\textrm{compact}: \mathbb{R}^{H \times W \times 3} \rightarrow \{1, \ldots, K\}^{N}$ maps an input image $\vo$ into a sequence of $N(N \le 16)$ discrete tokens $\vz$, each selected from a vocabulary of size $K$. 
The encoder architecture consists of three components: (1) a frozen DINOv3 model that extracts semantic patch representations, (2) a latent resampler with learnable query tokens, and (3) a finite scalar quantization layer~\cite{mentzer2023finite}. 

Specifically, the input image is patchified and encoded by the frozen DINOv3 model to obtain semantic representations. 
The initial latent tokens $\vz^0 \in \R^{N \times D}$ then act as learnable queries in a transformer decoder-based latent resampler ~\cite{vaswani2017attention}. 
In each decoder block, these latent tokens attend to the DINOv3 output patch tokens via cross-attention layers, effectively distilling high-level semantic cues from the pretrained representations. 
Because the vision foundation model has already abstracted away textures, lighting, and other low-level details, the cross-attention mechanism can selectively focus on semantic information—object identities, spatial layouts, and scene structure—that remains in the frozen features. 
The output of the latent resampler is then discretized using finite scalar quantization, yielding discrete latent tokens $\vz \in {1, \ldots, K}^{N}$. While such extreme compression inevitably discards fine-grained visual details, we hypothesize that these details are largely irrelevant for planning tasks, where object-level semantics and spatial relationships dominate decision-making.

\subsubsection{Generative decoding}
\label{sec:comppacttok_dec}
Direct pixel reconstruction from $N \le 16$ tokens is an ill-posed problem—the information bottleneck prevents the deterministic recovery of perceptual details, since diverse pixel-space manifestations can arise from identical semantic features.
To address this, we propose a generative decoding strategy that introduces an intermediate representation.
Our decoder $\mathcal{D}_\textrm{compact}: \{1, \ldots, K\}^{N} \rightarrow \{1, \ldots, K_{\psi}\}^{N_{\psi}}$
learns to generate latent tokens from a pretrained tokenizer $\mathcal{D}_{\psi} \circ \mathcal{E}_{\psi}$
~\cite{chang2022maskgit}, using our compact tokens $\vz$ as a condition. 
We refer to this pretrained tokenizer as the \textit{target tokenizer} because its tokens serve as intermediate targets that bridge our semantic representation to pixel space. 
Specifically, we employ the VQGAN from MaskGIT~\cite{chang2022maskgit}, which encodes images into hundreds of tokens ($N_{\psi} \gg N$, typically $N_{\psi} = 196$ for $224 \times 224$ images) capturing perceptual details omitted in our compact tokens.
This transforms the intractable decompression problem into a conditional generation task.

Specifically, we first convert an image $\vo$ into target tokens $\vz^{\psi} = \mathcal{E}_{\psi}(\vo) \in \{1, \ldots K_\psi\}^{N_{\psi}}$ using the pretrained tokenizer encoder, where $N_{\psi} \gg N$ (typically $N_\psi = 196$ for 224 $\times$ 224 images).
We then employ masked generative modeling~\cite{chang2022maskgit,yu2023language} to learn the mapping from $\vz$ to $\vz^{\psi}$, which offers significantly faster sampling than autoregressive models~\cite{gpt3,sun2024autoregressive}.
{During training, a random subset of the target tokens $\vz^{\psi}$ is masked, and the decoder learns to recover them using the compact tokens $\vz$ and the remaining unmasked tokens.}
The tokenizer training objective is defined to minimize the negative log-likelihood of the masked tokens $\vz^\psi$:
\begin{align}
\mathcal{L}_\textrm{tok} = - \underset{{z^\psi}}{\mathbb{E}} \big[ \log p(\vz^\psi | \vz, M(\vz^{\psi})) \big],
\end{align}
where $M(\cdot)$ represents the random masking.
\modelname is trained solely with the unmasking objective in latent space without pixel-level reconstruction, where the weights of the target tokenizers are not updated.
{During inference, $\mathcal{D}_\textrm{compact}$ begins with a fully masked sequence and iteratively unmasks them following the sampling scheme based on its prediction confidence~\cite{chang2022maskgit}.}
The compact tokens $\vz$ provide high-level semantic guidance throughout this process, while the generative model synthesizes plausible visual details consistent with these semantics.
The final reconstruction is obtained through the target decoder: $\hat{\vo} = (\mathcal{D}_\psi \circ \mathcal{D}_\textrm{compact} \circ \mathcal{E}_\textrm{compact}) (\vo)$.

In a nutshell, our \modelname tokenizer achieves extreme compression by preserving only high-level semantics in $N(N \le 16)$ discrete tokens, then using these as conditioning for a generative decoder that synthesizes plausible high-frequency details. 
This design aligns with our core hypothesis that effective planning requires not photorealistic world models, but compact representations of decision-critical information. 

\subsection{World model in \modelname latent space}
\label{sec:compact_lwm}
With our \modelname tokenizer defined, we can now train the world model formulated in Eq.~(\ref{eq:world}) directly in the $N$-token discrete latent space ($N \le 16$), as shown in Fig.~\ref{fig:overall}(b).
Given a dataset of observations and action sequences, we first encode all observations into compact latent tokens using \modelname tokenizer: $\vz_t = \mathcal{E}_\textrm{compact}(\vo_t)$. 
Similar to generative decoding, we use the masked generative modeling
to train the world model $f_{\phi}$.
The training objective is given by
\begin{align}
\mathcal{L}_\textrm{world} = - \underset{{z_t, a_t, z_{t+1}}}{\mathbb{E}} \big[ \log p(\vz_{t+1} | \vz_t, \va_t, M(\vz_{t+1})) \big].
\end{align}
The key advantage of this formulation is computational efficiency during planning. During model-predictive control, it can now perform rollouts using only $N(N \le 16)$ tokens per timestep, enabling planning latency that was previously intractable with hundred-length tokens.

Since the specific choice of world model architecture is orthogonal to our tokenizer design, any model capable of modeling discrete sequence distributions can be employed.
We explore two frameworks for learning the conditional distribution $p(\vz_{t+1} | \vz_t, \va_t)$. 
For navigation tasks, we follow an autoregressive framework following NWM~\cite{bar2025navigation}: at each step, the model predicts $\vz_{t+1}$ conditioned on a fixed-length history window of latents $\{\vz_{t-\tau}, \ldots, \vz_t\}$ and actions $\{\va_{t-\tau}, \ldots, \va_t\}$, implemented using a DiT-based architecture~\cite{peebles2023scalable}. 
To improve action conditioning, we randomly mask latent tokens in the history window during training.
For robotic manipulation on RoboNet~\cite{dasari2019robonet}, we employ a block-causal transformer that models multiple future frames simultaneously, predicting $\{\vz_{t+1}, \ldots, \vz_{t+K}\}$ in parallel while maintaining causal dependencies between frames. 

Both training schemes can be understood as discrete variants of the diffusion forcing~\cite{chen2024diffusion}: the navigation model learns to condition on partially masked context, while the parallel generation naturally implements diffusion forcing as frames at different unmasking stages provide varying levels of noisy conditioning.
This robust training improves planning accuracy without additional cost (ablation in Table~\ref{tab:ctx_masking_cost}; see the supplement for implementation details).

%% file: sec/4_experiment.tex
\section{Experiment}
\label{sec:experiment}

\subsection{Experimental Settings}
We evaluate \modelname across two key aspects: (1) \textbf{tokenization quality }through reconstruction metrics, and (2) \textbf{planning effectiveness} through action-conditioned world models in navigation and manipulation tasks. This dual evaluation validates our hypothesis that extreme compression preserves planning-critical information while enabling efficient decision-time planning.

\noindent \textbf{Task conductive.}
We evaluate \modelname on the following tasks:
(1) Image reconstruction: Reconstructing original images from compressed latent tokens.
(2) Goal-conditioned visual navigation: Given a context image and a navigation goal image, planning optimal paths by predicting future observations conditioned on actions.
(3) Action-conditioned video prediction: Generating future video frames conditioned on current visual input and actions.

\noindent \textbf{Dataset.}
We train \modelname on ImageNet-1K~\cite{deng2009imagenet}.
World models are trained on three navigation datasets following NWM~\cite{bar2025navigation}: RECON~\cite{shah2021rapid}, SCAND~\cite{karnan2022socially}, and HuRoN~\cite{hirose2023sacson}. For manipulation, we use RoboNet~\cite{dasari2019robonet}, which contains diverse robot interaction data across multiple environments.

\noindent \textbf{Tokenizer baselines.}
We compare \modelname against two baseline tokenizers: 
(1) {SD-VAE}~\cite{rombach2022high}: A continuous latent space tokenizer requiring 784 tokens per image ($28\times28$ spatial grid), representing the conventional approach used in state-of-the-art world models like NWM~\cite{bar2025navigation}. This establishes the computational cost we aim to reduce.
(2) {FlexTok}~\cite{bachmann2025flextok}: A recent discrete tokenizer supporting flexible token lengths (1-256 tokens). We evaluate it at 16 and 64 tokens to directly compare against \modelname's compression levels. 

\noindent \textbf{Evaluation.}
We employ task-appropriate metrics:
\begin{itemize}[leftmargin=*]
    \item \textbf{Reconstruction quality}: reconstruction FID (rFID)~\cite{parmar2022aliased} and Inception Score (IS)~\cite{salimans2016improved} on ImageNet validation set.
    \item \textbf{Planning accuracy}: Absolute Trajectory Error (ATE) and Relative Pose Error (RPE) for navigation tasks, measuring how closely planned trajectories match ground truth.
    \item \textbf{Action-relevancy}: Inverse Dynamics Model (IDM) performance using L1 error and coefficient of determination on predicted end-effector positions, validating that compact tokens preserve action-critical information.
    \item \textbf{Action-conditioned video prediction}: Action Prediction Error (APE), measured as the L1 error between the conditioning action and the action predicted by IDM from generated frames. This evaluates whether generated videos accurately reflect the dynamics induced by the actions.
    \item \textbf{Computational efficiency}: Planning latency during model-predictive control.
\end{itemize}

\noindent \textbf{Implementation details.} We use DINOv3-B~\cite{simeoni2025dinov3} for pretrained vision encoder in $\mathcal{E}_\textrm{compact}$. For the generative decoder, we use VQGAN from MaskGIT~\cite{chang2022maskgit} as the target tokenizer $\mathcal{D}_\psi \circ \mathcal{E}_\psi$. For detailed hyperparameters for training and model architecture, we refer readers to the supplement.

\subsection{Tokenizer evaluation and ablations}
\input{table/table_recon}
\noindent \textbf{Reconstruction performance.}
Table~\ref{tab:vs_2d_tok} presents the reconstruction quality of each tokenizer measured by rFID and IS. 
MaskGIT-VQGAN~\cite{chang2022maskgit}, which is used as a target tokenizer $\mathcal{D}_{\psi} \circ \mathcal{E}_{\psi}$ in \modelname, serves as a baseline for reconstruction fidelity since \modelname relies on $\mathcal{D}_{\psi}$ for final pixel reconstruction.
The results demonstrate that \modelname can attain reconstruction performance comparable to recent state-of-the-art tokenizers while achieving extreme compression rate.
Interestingly, \modelname outperforms MaskGIT-VQGAN~\cite{chang2022maskgit} in IS, suggesting that our semantic encoder better preserves perceptually-relevant features that may be overlooked by VQGAN's pixel-focused reconstruction objective.

\input{table/table_encoder_ablation}

\begin{figure}[t!]
    \centering
    \includegraphics[width=0.48\textwidth]{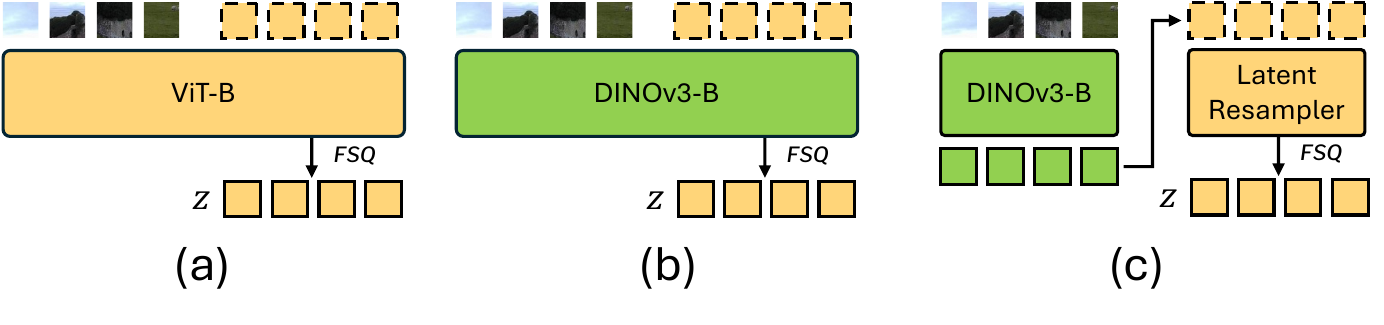}
    \vspace{-8mm}
    \caption{\textbf{\modelname encoder $\mathcal{E}_\textrm{compact}$ architecture variation.} (a) ViT (scratch)+ [\texttt{REG}]: Initial latent tokens are concatenated to the input patch tokens. This design follows previous transformer-based image tokenizers~\cite{yu2024image,bachmann2025flextok,yu2021vector}. (b) DINOv3~\cite{simeoni2025dinov3} + [\texttt{REG}]: Similar to (a), but encoder is initialized with Dinov3. (c) DINOv3~\cite{simeoni2025dinov3} + latent resampler: latent resampler and Dinov3 initialized encoder. Dino and ViT are updated during training in these variants.}
    \label{fig:ablation}
\end{figure}

\noindent \textbf{Ablation of encoder design choices.}
We report ablation studies on our final \modelname tokenizer design choices in Table~\ref{tab:encoder_ablation}. 
We compare three encoder architecture variants illustrated in Fig.~\ref{fig:ablation}.  
The results demonstrate that leveraging frozen features from the semantic encoder is a key design choice enabling extremely compact tokenization while maintaining strong reconstruction quality. Notably, full finetuning of DINOv3-B results in significantly worse rFID of 5.22. 
We conjecture that finetuning the vision foundation model shifts its representations toward reconstruction-oriented features, causing the compact latent tokens to lose the high-level semantic information that is crucial for our generative decoding approach. Without rich semantic conditioning, the generative decoder cannot synthesize plausible fine-grained details, leading to degraded reconstruction quality. 

\noindent \textbf{Ablation of generative decoding.}
Table~\ref{tab:encoder_ablation} also shows the ablation studies for the generative decoding. 
To check the effectiveness of the generative decoding, we replaced the $\mathcal{D}_\text{compact}$ with the single-step feedforward decoder, which leads to the severe degradation in reconstruction quality. This results highlights the $\mathcal{D}_\text{compact}$'s crucial role in synthesizing fine-grained features, as compact latent token only preserves high-level semantic features.

\subsection{Characterizing \modelname latent tokens}

\begin{figure}[t!]
    \centering
    \includegraphics[width=0.48\textwidth]{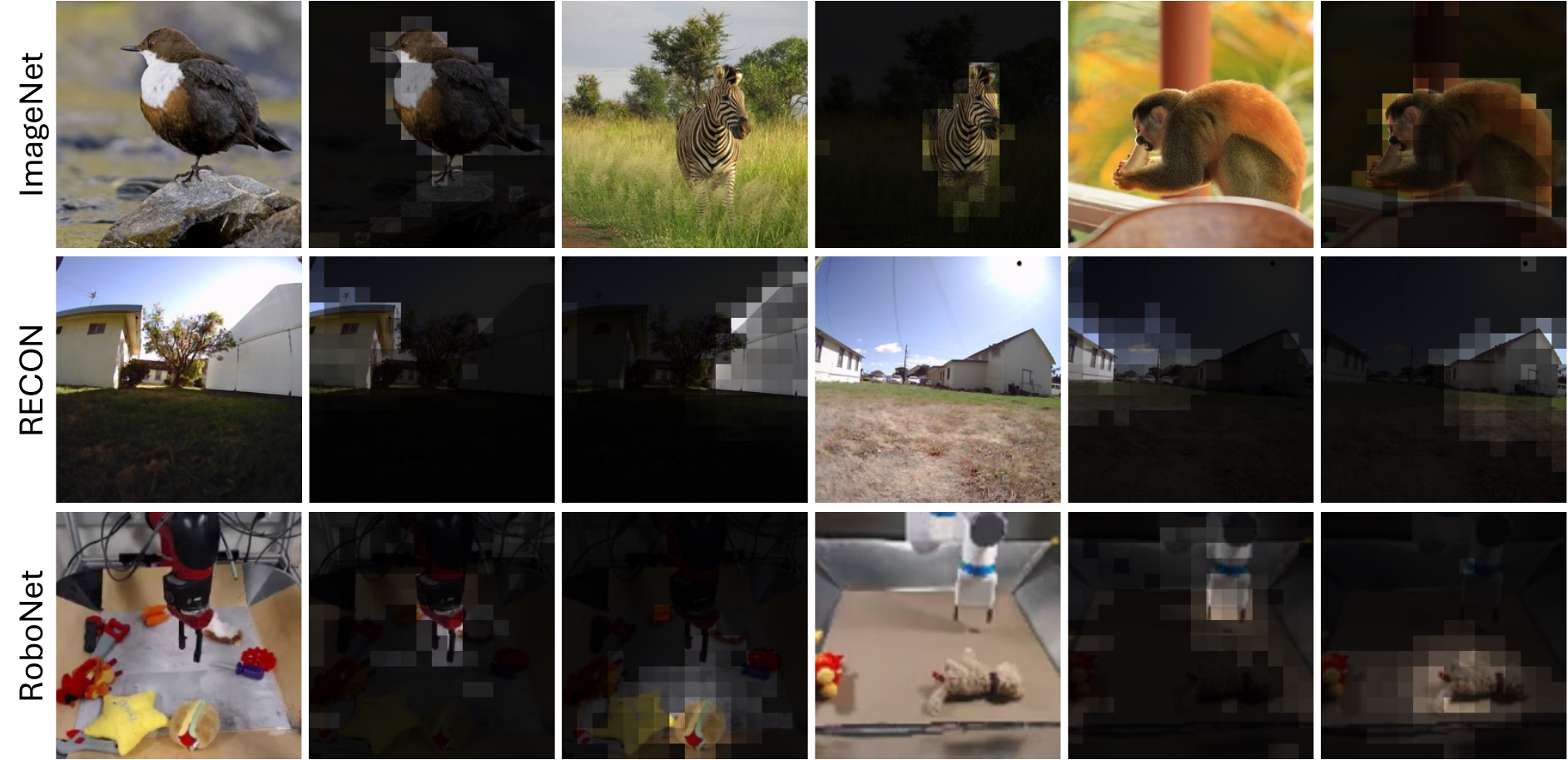}
    \vspace{-6mm}
    \caption{\textbf{Attention visualization for compact latent token in latent resampler.} Brighter the color, higher the attention score.}
    \label{fig:qual_atten}
\end{figure}

\noindent \textbf{\modelname tokens capture modular scene elements.}
Fig.~\ref{fig:qual_atten} visualizes the attention maps from the latent resampler across ImageNet, RECON, and RoboNet. 
We observe that each compact latent token attends to coherent, semantically meaningful regions within the image, effectively capturing modular object-level elements.
This compositional structure emerges naturally from frozen DINO features used for latent resampling, as they already encode object-centric representations.
Specifically, in ImageNet, tokens attend to distinct objects (animals); in RECON, they focus on structural elements like buildings; in RoboNet, they isolate the end-effector and manipulation targets.
Rather than distributing information uniformly across tokens, \modelname learns to allocate each token to semantically coherent scene components without explicit supervision.

\input{table/table_idm}
\noindent \textbf{Modular latents benefit planning.}
We now validate whether this modular structure translates to improved planning performance.
Since planning requires understanding how actions induce state transitions, we use an Inverse Dynamics Model (IDM) as a proxy to evaluate whether our compact tokens preserve dynamics-relevant information: if we can train an IDM that accurately predicts actions from consecutive frames on latent representations, it indicates that the latents capture the essential state changes for control tasks.
For details of IDM implementation, please refer to the supplementary material.
As shown in Table~\ref{tab:idm}, the IDM trained on \modelname latents achieves superior performance compared to MaskGIT-VQGAN despite using 16× fewer tokens. 
For this experiment, we finetune \modelname (pretrained on ImageNet) on RoboNet to adapt the representations to the manipulation domain.
This performance gap reveals a fundamental difference in token allocation: \modelname's modular tokens naturally capture dynamic objects—the end-effector and manipulation targets—whose state transitions encode action information. 
By allocating tokens to semantically coherent objects rather than fixed spatial regions, our compact representation better captures the state changes induced by actions, benefiting downstream planning tasks.

\input{table/table_planning_performance}
\begin{figure}[t!]
    \centering
    \includegraphics[width=0.48\textwidth]{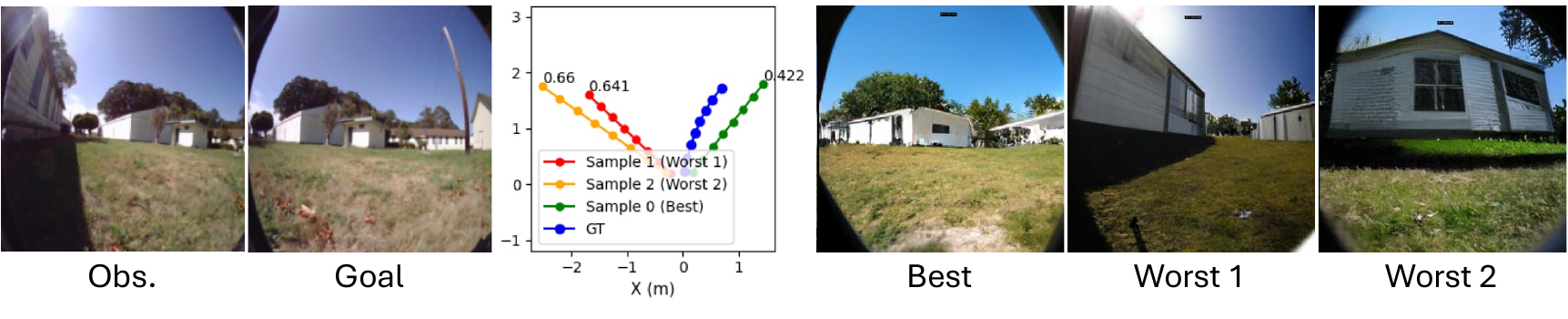}
    \vspace{-8mm}
    \caption{\textbf{Qualitative results of planning with the proposed \modelname.} 
    {Best and Worst 1\&2 denote the final rollouts corresponding to the simulated trajectories with the minimum and maximum cost, respectively.}}
    \label{fig:qual_plan}
\end{figure}
\input{table/table_ctx_making_cost}
\subsection{Planning in \modelname latent space}
\noindent \textbf{Planning performance.} 
Table~\ref{tab:planning_performance} presents planning results for goal-conditioned visual navigation on the RECON dataset. Our \modelname achieves approximately \textit{40× reduction in planning latency} while maintaining comparable planning accuracy to the SD-VAE baseline that uses 784 tokens. 
Notably, NWM trained with \modelname (using only 16 or 8 tokens) outperforms the FlexTok-based model at both 16 and 64 token configurations.
This performance advantage stems from our encoder's design: $\mathcal{E}_\textrm{compact}$ leverages frozen semantic features from pretrained vision foundation models for latent token resampling, ensuring that semantically meaningful features—spatial relationships, object configurations, and scene structure—are prioritized over reconstruction-oriented details like textures or lighting (Fig.~\ref{fig:qual_atten}). This semantic prioritization enables the world model to focus on action-relevant state transitions, as further validated by our inverse dynamics modeling results (Tab.~\ref{tab:idm}).

\noindent \textbf{Qualitative examples.} 
Fig.~\ref{fig:qual_plan} presents qualitative planning results with \modelname. While fine-grained visual details such as textures and shadows are synthesized rather than reconstructed, the rollouts accurately preserve planning-critical information needed for effective goal-reaching: spatial layout and object positions.

\noindent \textbf{In-depth analysis on the effect of history masking, tokenizer, and cost function.}
Table~\ref{tab:ctx_masking_cost} analyzes three key design choices that contribute to our method's effectiveness.
(\textit{Left}) \textit{History masking}: Results demonstrate that incorporating history masking during world model training improves planning accuracy, validating that masking encourages robust temporal dependency learning.
(\textit{Middle}) \textit{Cost function}: We compare planning accuracy when the cost is computed in pixel space (LPIPS) versus latent space (L1 distance between $\vz$). 
While LPIPS achieves marginally better planning accuracy, computing distances in latent token space offers substantial computational benefits—achieving nearly 80× speedup in planning time compared to SD-VAE-based planning (Tab.~\ref{tab:planning_performance}).
(\textit{Right}) \textit{Frozen encoder}: Similar to the reconstruction quality degradation in Tab.~\ref{tab:encoder_ablation}, updating the vision encoder during tokenizer training leads to degraded planning performance as well. Fine-tuning the vision foundation model shifts its representations toward reconstruction objectives, causing compact tokens to lose high-level semantic information essential for planning.

\input{table/table_robonet}
\begin{figure}[t!]
    \centering
    \includegraphics[width=0.48\textwidth]{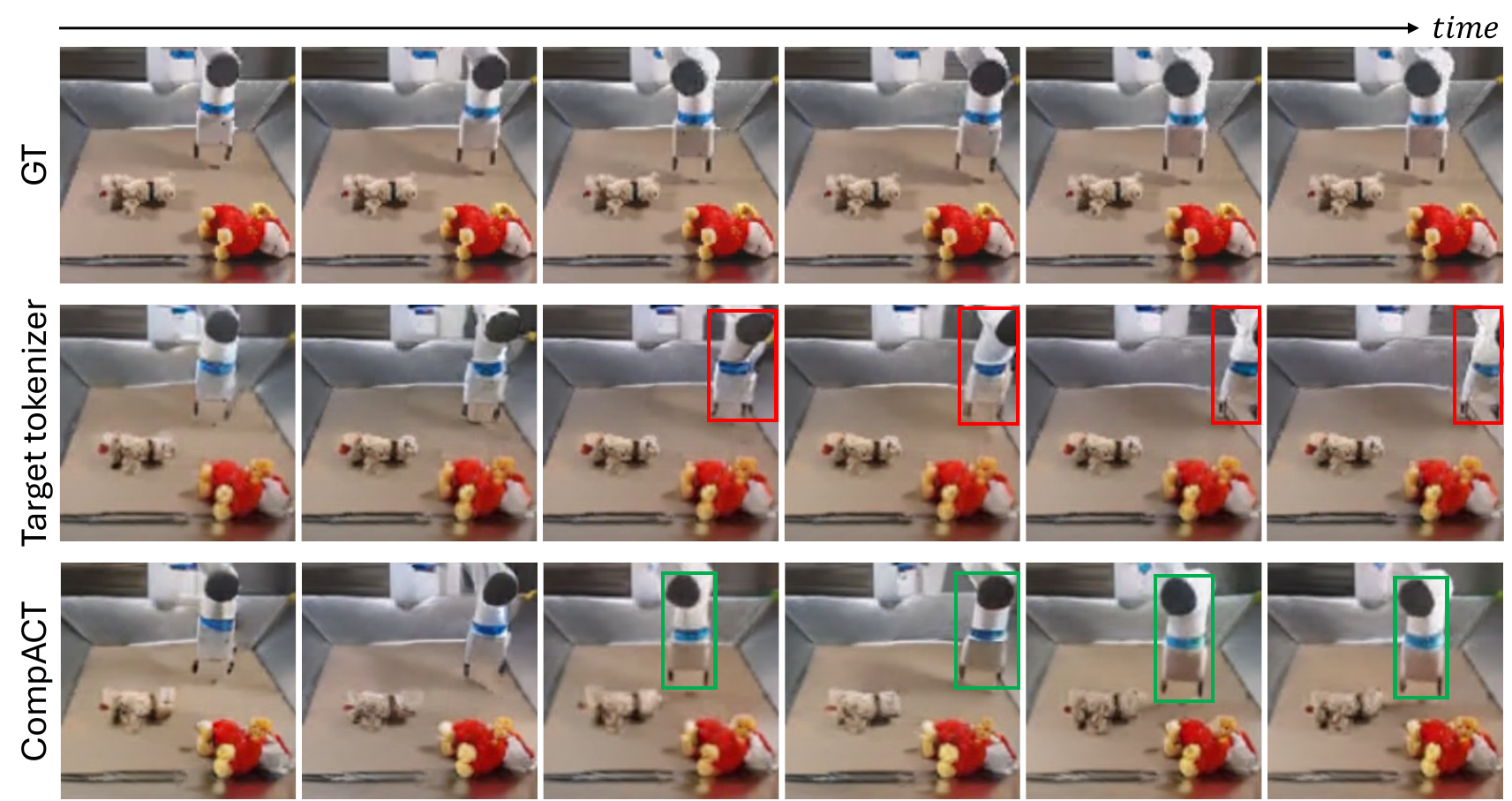}
    \vspace{-6mm}
    \caption{\textbf{Qualitative results of action-conditioned video generation.} Red and green boxes indicate incorrect and correct end-effector positions, respectively.}
    \label{fig:qual_action}
    \vspace{-3mm}
\end{figure}
\noindent \textbf{Action conditioned video prediction.}
To further validate that \modelname's modular latent representation preserves action-relevant information, we evaluate action-conditioned video generation on RoboNet~\cite{dasari2019robonet}. Table~\ref{tab:robonet} shows that \modelname achieves 3× lower action prediction error (APE) compared to the 256-token baseline, while providing 5.2× faster generation. 
APE measures how accurately an IDM (trained as in Tab.~\ref{tab:idm}) can recover the conditioning action from generated video frames—a metric that directly evaluates whether the world model captures action-driven dynamics.
The substantial improvement in APE validates our hypothesis that \modelname's modular tokens, which naturally attend to dynamic objects like end-effector and manipulation targets (Fig.~\ref{fig:qual_atten}), are inherently better suited for modeling action-conditioned state transitions.
Qualitative results (Fig.~\ref{fig:qual_action}) further support this: videos generated from \modelname latents maintain consistent action-driven end-effector movements, while the target tokenizer fails to preserve these dynamics.

%% file: table/table_recon.tex
\begin{table}
\centering
\caption{
\textbf{Reconstruction performance of \modelname on ImageNet validation split.}
Metrics are computed using open-sourced checkpoints. rFID is measured using clean-fid~\cite{parmar2022aliased}. $\dagger$: Measured using 16 tokens.
}
\vspace{-2mm}
\resizebox{0.85\columnwidth}{!}{
\begin{tabular}{l|cc|cc}
\toprule
 \textbf{Model} &  \textbf{Type} & \textbf{\#tok} & \textbf{rFID} $\downarrow$ & \textbf{IS} $\uparrow$\\  \midrule
 SD-VAE~\cite{rombach2022high} &  cont. & 1024 & 0.64 & 223.8 \\
 MaskGIT-VQGAN~\cite{chang2022maskgit} & disc. & 256 & 1.83 & 186.7 \\
 TA-TiTok-VQ~\cite{kim2025democratizing} & disc. & 32 & 3.95 & 219.6\\
 TA-TiTok-KL~\cite{kim2025democratizing} & cont. & 32 & 1.93 & 222.0\\
 FlexTok~\cite{bachmann2025flextok} & disc. & 1-256 & 5.60$^\dagger$ & 114.9 \\  \midrule
 \rowcolor{lightgray!50}\modelname & disc. & 16 & 2.40 & 209.0 \\
 \rowcolor{lightgray!50}\modelname & disc. & 8 & 3.21 & 207.5 \\
 \bottomrule
\end{tabular}
}
\label{tab:vs_2d_tok}
\end{table}

%% file: table/table_encoder_ablation.tex
\begin{table}[t]
\centering
\caption{\textbf{Ablation on \modelname tokenizer.} rFID is measured on ImageNet~\cite{deng2009imagenet} validation split using clean-fid~\cite{parmar2022aliased}.}
\vspace{-2mm}
\small
\resizebox{\columnwidth}{!}{
\begin{tabular}{ll|cc}
\toprule
\multicolumn{2}{l|}{\textbf{Configuration}} & \textbf{rFID$\downarrow$} & \textbf{\#tok} \\
\midrule
\multicolumn{2}{l|}{Target tokenizer ($\mathcal{D}_{\psi} \circ \mathcal{E}_{\psi}$)~\cite{chang2022maskgit}} & 1.83 & 256 \\ 
\rowcolor{lightgray!50}\multicolumn{2}{l|}{DINOv3-B (frozen) + latent resampler} & 2.40 & 16 \\\midrule
$\mathcal{D}_\textrm{compact}$ & w/o generative decoding & 28.80 & 16 \\ \midrule
\multirow{3}{*}{$\mathcal{E}_\textrm{compact}$} & ViT-B (trained from scratch) + [REG] & 7.28 & 16 \\
 & DINOv3-B (finetuned) + [REG] & 4.51 & 16 \\
 & DINOv3-B (finetuned) + latent resampler & 5.22 & 16 \\
\bottomrule
\end{tabular}
}
\label{tab:encoder_ablation}
\end{table}

%% file: table/table_idm.tex
\begin{table}[t]
\centering
\caption{\textbf{Performance of Inverse Dynamics Model (IDM) trained with different tokenizers on RoboNet~\cite{dasari2019robonet}.} L1 error and $R^2$ are measured between ground truth and predicted end effector position.}
\vspace{-2mm}
\scalebox{0.75}{
\begin{tabular}{lc|cc}
\toprule
\textbf{Tokenizer} & \textbf{\#tok} & \textbf{L1 err$\downarrow$} & \textbf{$R^2$ $\uparrow$}\\ \midrule
Target tokenizer ($\mathcal{D}_{\psi} \circ \mathcal{E}_{\psi}$)~\cite{chang2022maskgit} & 256& 0.093& 0.684\\
\rowcolor{lightgray!50}\modelname & 16& 0.091& 0.716\\
\bottomrule
\end{tabular}
}
\label{tab:idm}
\end{table}

%% file: table/table_planning_performance.tex
\begin{table}
\centering
\caption{\textbf{Planning performance of NWM on RECON benchmark with different tokenizers.} Latency (sec) represents single trajectory optimization time using a single RTX 6000 ADA GPU.}
\vspace{-2mm}
\small
\resizebox{\columnwidth}{!}{
    \begin{tabular}{lc|cc|cc|c}
    \toprule
    \multirow{2}{*}{\textbf{Tokenizer}} & \multirow{2}{*}{\#tok} & \multicolumn{2}{c|}{\textbf{RECON}} & \multicolumn{2}{c|}{\textbf{SCAND}} & \multirow{2}{*}{\textbf{Latency}$\downarrow$}\\
     &  & ATE$\downarrow$ & RPE$\downarrow$ & ATE$\downarrow$ & RPE$\downarrow$ &  \\
    \midrule
    SD-VAE~\cite{rombach2022high} & 784 &  \textbf{1.262} & \textbf{0.354} & \textbf{1.065}& \textbf{0.291} &178.78 \\ \midrule
    FlexTok~\cite{bachmann2025flextok} & 64 &  1.484 & 0.400 & 1.578& 0.378 & 16.68\\
    FlexTok~\cite{bachmann2025flextok} & 16 &  1.625 & 0.446 & 1.503& 0.362& 14.48 \\ \midrule
    \rowcolor{lightgray!50}\modelname & 16 & \underline{1.330} & \underline{0.390} & \underline{1.358}& \underline{0.336} & \underline{5.78} \\
    \rowcolor{lightgray!50}\modelname & 8 & 1.373 & 0.401 &  1.391&  0.346& \textbf{4.83}\\
    \bottomrule
    \end{tabular}
}
\label{tab:planning_performance}
\end{table}

%% file: table/table_ctx_making_cost.tex
\begin{table}[t]
\caption{\textbf{Effect of design choices in terms of planning accuracy on RECON.} (\textit{Left}): Effect of the history masking in the world model $f_\theta$ (Sec.~\ref{sec:compact_lwm}). (\textit{Middle}): Comparison between the different cost function (Sec.~\ref{sec:formulation}). (\textit{Right}): Effect of the freezing vision encoder during tokenizer training (Sec~\ref{sec:comppacttok_enc}).}
\vspace{-2mm}
\centering
\begin{minipage}{0.22\columnwidth}
    \centering
    \scalebox{0.75}{
    \begin{tabular}{c|c}
    \toprule
    \textbf{Hist. mask} & \textbf{ATE}\\ \midrule
    \rowcolor{lightgray!50}\checkmark & 1.330 \\
      & 1.480 \\
    \bottomrule
    \end{tabular}
    }
\end{minipage}
\hfill
\begin{minipage}{0.28\columnwidth}
    \centering
    \scalebox{0.75}{
    \begin{tabular}{c|cc}
    \toprule
     \textbf{Cost} & \textbf{ATE} & \textbf{Latency} \\ \midrule
     \rowcolor{lightgray!50}Pixel & 1.330 & 5.78sec \\
     Latent & 1.379 & 2.15sec \\
    \bottomrule
    \end{tabular}
    }
\end{minipage}
\hfill
\begin{minipage}{0.3\columnwidth}
    \centering
    \scalebox{0.75}{
    \begin{tabular}{c|c}
    \toprule
     \textbf{Frozen} & \textbf{ATE} \\ \midrule
     \rowcolor{lightgray!50}\checkmark & 1.330 \\
      & 1.472\\
    \bottomrule
    \end{tabular}
    }
\end{minipage}
\label{tab:ctx_masking_cost}
\end{table}

%% file: table/table_robonet.tex
\begin{table}[t]
\caption{\textbf{Action-conditioned video prediction results on RoboNet~\cite{dasari2019robonet}.} Action prediction error is measured using IDM reported in Tab~\ref{tab:idm}. Latency is measure for when generating next 14 frames on a single RTX 6000 ADA GPU.}
\vspace{-2mm}
\centering
\scalebox{0.75}{
\begin{tabular}{lc|ccc}
\toprule
\textbf{Model} & \textbf{\#tok} & \textbf{APE} & \textbf{Latency (sec)} \\ \midrule
Target tokenizer ($\mathcal{D}_{\psi} \circ \mathcal{E}_{\psi}$)~\cite{chang2022maskgit} & 256& 0.3383 & 3.826 \\
\rowcolor{lightgray!50}\modelname & 16& 0.1122 & 0.740 \\
\bottomrule
\end{tabular}
}
\label{tab:robonet}
\end{table}

%% file: sec/5_conclusion.tex
\section{Conclusion}
In this work, we present \modelname, a compact tokenizer that achieves extreme compression by representing images with only 16 or 8 discrete tokens while preserving planning-critical information.
The key insight enabling this compression is our use of frozen vision foundation models as the encoder backbone: by leveraging pretrained semantic representations, our tokenizer naturally prioritizes high-level spatial and semantic features over reconstruction-oriented details.
We demonstrate that world models trained in this compact latent space outperform baselines requiring larger token counts while achieving a 40× speedup in planning, validating our hypothesis that effective planning requires semantic abstraction rather than photorealistic reconstruction.

%% file: sec/X_suppl.tex
\clearpage
\maketitlesupplementary
\setcounter{section}{0}
\renewcommand{\thesection}{\Alph{section}}
\noindent
This supplementary material provides detailed implementation specifics and additional experimental results for \modelname. We organize the content as follows:
\begin{itemize}
    \item \textbf{Sec.~\ref{sec:supp_planning_sufficiency}} provides theoretical framework for the planning sufficiency of compact tokens.
    \item \textbf{Sec.~\ref{sec:supp_backbone}} presents cross-backbone ablation for \modelname.
    \item \textbf{Sec.~\ref{sec:supp_latency_breakdown}} provides a planning latency breakdown.
    \item \textbf{Sec.~\ref{sec:supp_rm_closed_loop}} evaluates closed-loop robot arm manipulation on Robomimic~\cite{robomimic2021}.
    \item \textbf{Sec.~\ref{sec:supp_tok}} describes the complete \modelname tokenizer architecture, training procedure, and hyperparameters.
    \item \textbf{Sec.~\ref{sec:supp_idm}} details the Inverse Dynamics Model (IDM) used for action prediction experiments on RoboNet~\cite{dasari2019robonet}.
    \item \textbf{Sec.~\ref{sec:supp_wm_arch}} presents the world model architectures for both navigation (autoregressive with fixed history window) and manipulation tasks (block-causal parallel prediction).
    \item \textbf{Sec.~\ref{sec:supp_planning}} explains the Cross-Entropy Method used for navigation planning.
    \item \textbf{Sec.~\ref{sec:supp_qual}} provides additional qualitative results including reconstruction examples, attention visualizations, and planning rollouts.
    \item \textbf{Sec.~\ref{sec:supp_eff}} provides comparison between tokenizers in terms of planning efficiency.
    \item \textbf{Sec.~\ref{sec:supp_scale_up}} presents that with compact latent tokens we can scale up the world model while remaining efficient.
\end{itemize}

\input{sec/supp/planning_suff}

\input{sec/supp/other_backbone}
\input{sec/supp/latency_breakdown}
\input{sec/supp/robomimic_closed_loop}

\begin{table}[!t]
\small
\centering
\caption{
\textbf{Training hyperparameters for \modelname.}
}
\vspace{-3mm}
\scalebox{1.0}{
\begin{tabular}{l|c}
 \toprule
 hparams & \modelname \\ \midrule
optimizer & AdamW \\ 
$\beta_1$ & 0.9  \\
$\beta_2$ & 0.999 \\
weight decay & 0.01 \\ 
 
 lr  & 0.0001\\
 lr scheduling & cosine \\
 lr warmup steps & 10K \\ 

 batch size (ImageNet) & 512 \\
 training steps (ImageNet) & 500K \\
 training steps (RoboNet finetuning~\cite{dasari2019robonet}) & 256 \\
 training steps (RoboNet finetuning~\cite{dasari2019robonet}) & 100K  \\ \bottomrule
\end{tabular}
}
\label{tab:tok_training_hparams}
\end{table}

\begin{table}[!t]
\small
\centering
\caption{
\textbf{Model architecture hyperparameters for \modelname.}
}
\vspace{-3mm}
\scalebox{1.0}{
\begin{tabular}{l|cc}
 \toprule
 hparams & Latent resampler & $\mathcal{D}_\text{compact}$ \\ \midrule
depth & 5 & 16 \\ 
dim & 768 & 1024 \\
MLP dim & 3072 & 4096\\
heads & 8 & 8 \\
\bottomrule
\end{tabular}
}
\label{tab:tok_model_hparams}
\end{table}

\section{Details of \modelname tokenizer}\label{sec:supp_tok}
Tab.~\ref{tab:tok_training_hparams} and Tab.~\ref{tab:tok_model_hparams} summarize the training and model architecture hyperparameters of \modelname, respectively.

\noindent \textbf{Tokenizer architecture.}
We use frozen DINOv3-B~\cite{simeoni2025dinov3} in the encoder $\mathcal{E}_\textrm{compact}$.
In DINOv3-B, we re-initialized the last layer normalization's affine parameters (weight and bias to 1 and 0, respectively). We found that using pretrained affine parameters as-is results in codebook collapse during training, since the output of $\mathcal{E}_\textrm{compact}$ has specific statistics when using pretrained layernorm affine parameters.
For the latent resampler, we use the 5 transformer decoder blocks similar to DETR~\cite{carion2020end} and Perceiver~\cite{jaegle2021perceiver}. Number of learnable queries determines the number of latent tokens used in \modelname.
We use the Finite Scalar Quantization~\cite{mentzer2023finite} (FSQ) for discretizing the output of latent resampler. Levels per channel are set to $[8, 8, 8, 5, 5, 5]$, which is the recommended level configuration to construct approximately $2^{16}$ size codebook.
\begin{figure}[t!]
    \centering
    \includegraphics[height=9cm]{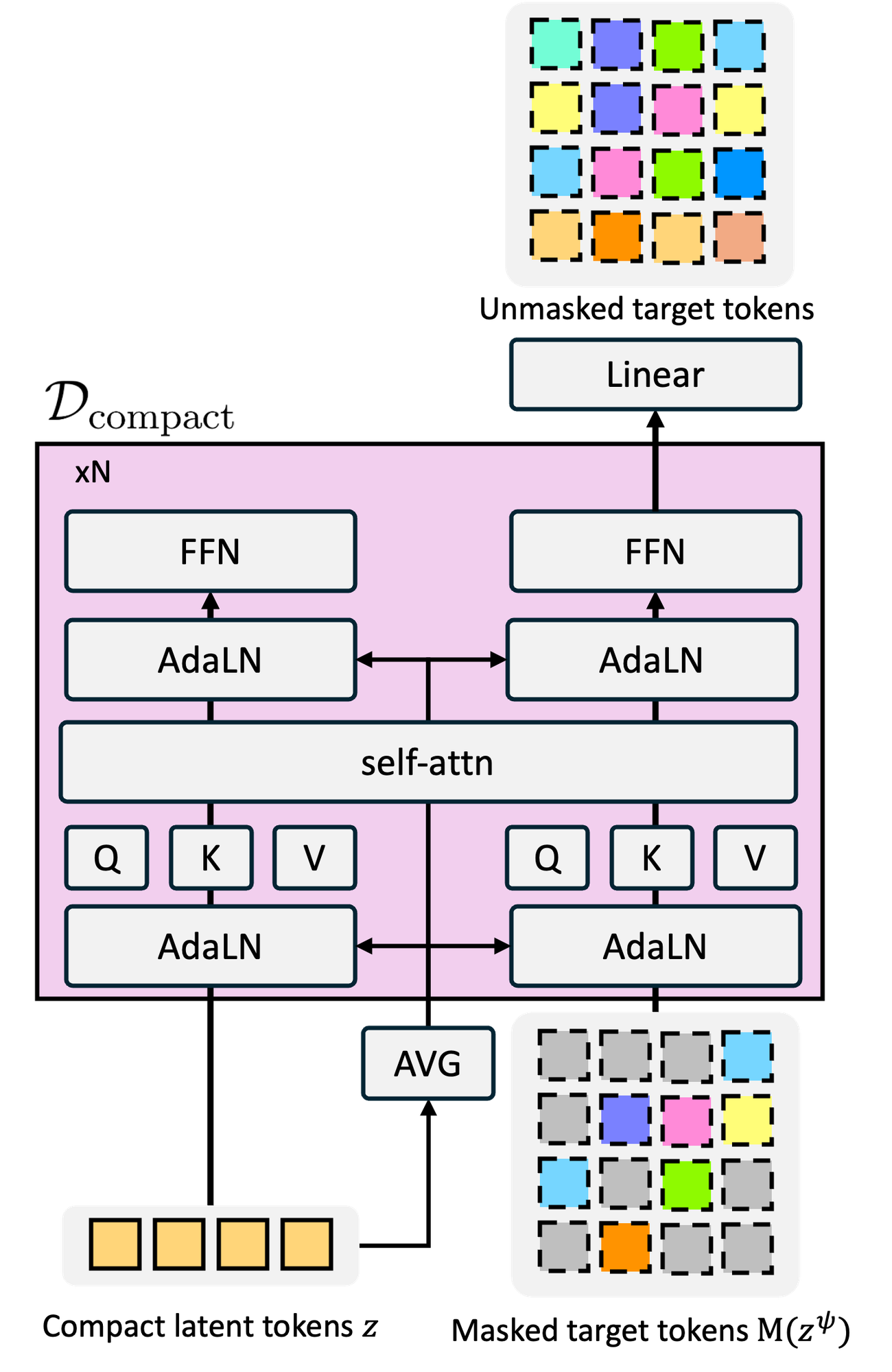}
\vspace{-3mm}
    \caption{\textbf{$\mathcal{D}_\text{compact}$ architecture.} 
The proposed $\mathcal{D}_\text{compact}$ is based on MM-DiT~\cite{esser2024scaling}, a DiT~\cite{peebles2023scalable} variant designed for multimodal input processing. 
The architecture consists of two parallel processing streams: one for compact latent tokens $\vz$ and another for target latent tokens $\vz^\psi$, which are fused through self attention over the concatenated token sequence. For the overall depiction of \modelname, we refer to Fig.~2 in the main paper.}
    \label{fig:dcompact}
\end{figure}
For decoder $\mathcal{D}_\textrm{compact}$, we use the MM-DiT~\cite{esser2024scaling}, which takes compact latent as a condition.
Fig.~\ref{fig:dcompact} depicts the details architecture of the decoder $\mathcal{D}_\textrm{compact}$, which is shown in simplified form in Fig.~2 in the main paper.
$\mathcal{D}_\textrm{compact}$ is trained from the scratch.
Proposed \modelname comprises 775M parameters in total, including $\mathcal{E}_\textrm{compact}$, $\mathcal{D}_\textrm{compact}$, and $\mathcal{D}_\psi$.

\noindent \textbf{Masked generative modeling.}
As noted in Sec.~3.2.2 of the main paper, $\mathcal{D}_\textrm{compact}$ is formulated as a masked generative model. During training, we randomly sample the mask ratio from $(0, 1]$ following the cosine masking schedule~\cite{chang2022maskgit}.
During inference, we decode target latent tokens $z^{\psi}$ by progressively unmasking them from a fully masked sequence, using compact latent tokens $z$ as conditioning. In each iteration, a fraction of predictions with high confidence is accepted, while remaining tokens are re-masked. We follow the same cosine masking schedule during inference.

\noindent \textbf{Dataset.}
\modelname tokenizer is trained on ImageNet-1K~\cite{deng2009imagenet}, on $224 \times 224$ (for fair comparison in navigation experiment following NWM~\cite{bar2025navigation}) and $256 \times 256$ resolution.
For the augmentation, we used random crop and random horizontal flip.
We additionally finetune the tokenizer for the experiment with RoboNet~\cite{dasari2019robonet} since robot data has a significant domain gap compared to the ImageNet data. For RoboNet finetuning we used center cropped $256 \times 256$ resolution images.

\noindent \textbf{Training and inference.}
For ImageNet~\cite{deng2009imagenet} pretraining, \modelname is trained for 500K steps with batch size of 512. We use the AdamW~\cite{loshchilov2017decoupled} with $1e-4$ learning rate. $\beta_1$ and $\beta_2$ are set to $0.9$ and $0.999$, respectively.
During inference, sampling step in $\mathcal{D}_\textrm{compact}$ is set to 16.

\begin{figure}[t!]
    \centering
    \includegraphics[height=6cm]{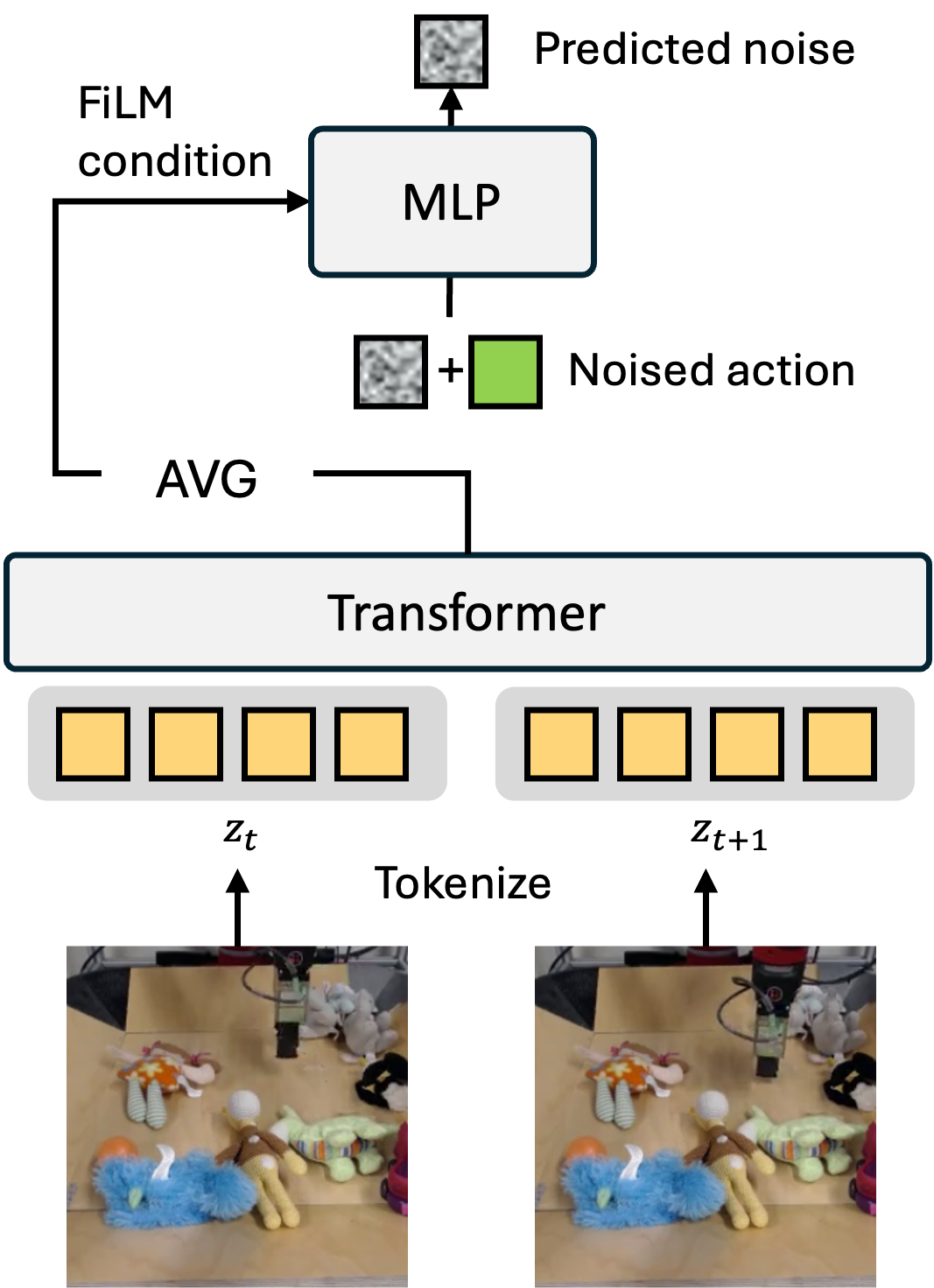}
    \caption{\textbf{Inverse Dynamics Model (IDM) architecture.} 
    Consecutive frames are tokenized and processed through a transformer-based frame encoder, which produces a single conditioning vector via average pooling. This vector conditions an action denoiser implemented as a diffusion policy~\cite{chi2023diffusionpolicy}, which predicts the action taken between the two frames.
    }
    \label{fig:idm_arch}
\end{figure}

\begin{table}
\centering
\caption{
\textbf{Model architecture hyperparameters for IDM.}
}
\vspace{-3mm}
\begin{minipage}{0.49\columnwidth}
\centering
\scalebox{0.85}{
    \begin{tabular}{l|c}
    \toprule
    hparams & Frame encoder \\ \midrule
    depth & 4  \\ 
    dim & 512 \\
    MLP dim & 2048 \\
    heads & 8 \\ 
    \#params & 13.5M \\ \bottomrule
    \end{tabular}
}
\end{minipage}
\hfill
\begin{minipage}{0.49\columnwidth}
\centering
\scalebox{0.85}{
    \begin{tabular}{l|c}
    \toprule
    hparams & Action denoiser \\ \midrule
    \# linear & 4  \\ 
    hidden dim & 512 \\
    diffusion & DDPM~\cite{ho2020denoising} \\
    timesteps & 1000 \\
    \#params & 3.6M \\ \bottomrule
    \end{tabular}
}
\end{minipage}
\label{tab:idm_model_hparams}
\end{table}

\section{Details of inverse dynamics model (IDM)}
\label{sec:supp_idm}
Fig.~\ref{fig:idm_arch} and Tab.~\ref{tab:idm_model_hparams} present the model architecture and hyperparameters of IDM, respectively.

\noindent \textbf{Architecture.}
The Inverse Dynamics Model (IDM) used for RoboNet experiments consists of two modules: a frame encoder and an action denoiser.
Given latent tokens from two consecutive frames, the frame encoder first processes the concatenated token sequence through a 4-layer transformer encoder.
The target tokenizer~\cite{chang2022maskgit} requires processing 512 tokens, while \modelname requires only 32 tokens—a 16$\times$ reduction in sequence length.
The frame encoder output is average-pooled into a single vector, which serves as the conditioning signal for the action denoiser.
The action denoiser is a multi-layer perceptron (MLP) with 4 linear layers (hidden dimension 512), SiLU activation~\cite{hendrycks2016gaussian}, and FiLM conditioning~\cite{perez2018film} between each layer.
The denoiser takes a noisy 5-dimensional action and predicts the applied noise.
Following diffusion policy~\cite{chi2023diffusionpolicy}, we use DDPM~\cite{ho2020denoising} with a squared cosine schedule~\cite{nichol2021improved}.

\noindent \textbf{Dataset.}
IDM is trained on RoboNet~\cite{dasari2019robonet}. Frame pairs are randomly sampled at each iteration. Frames are pre-encoded using tokenizer (\modelname or target tokenizer) with $256 \times 256$ resolution center-cropped images.
250 episodes are used for test.

\noindent \textbf{Training and inference.}
IDM is trained for 100K steps with batch size of 128. We use the AdamW~\cite{loshchilov2017decoupled} with $1e-4$ learning rate. $\beta_1$ and $\beta_2$ are set to $0.9$ and $0.999$, respectively.
During inference, we use 1000 diffusion timesteps for sampling.

\noindent \textbf{Evaluation.}
We evaluate IDM performance using two metrics computed on predicted end-effector positions. 
First, we measure L1 error between the predicted and ground truth end-effector positions. 
Second, we compute the coefficient of determination (R$^2$), which measures how well the model explains the variance in end-effector movements. R$^2$ ranges from 0 to 1, where higher values indicate better prediction. For instance, R$^2 = 0.9$ means that model can explain 90\% of the variance in end-effector movements.

\begin{figure}[t!]
    \centering
    \includegraphics[height=9cm]{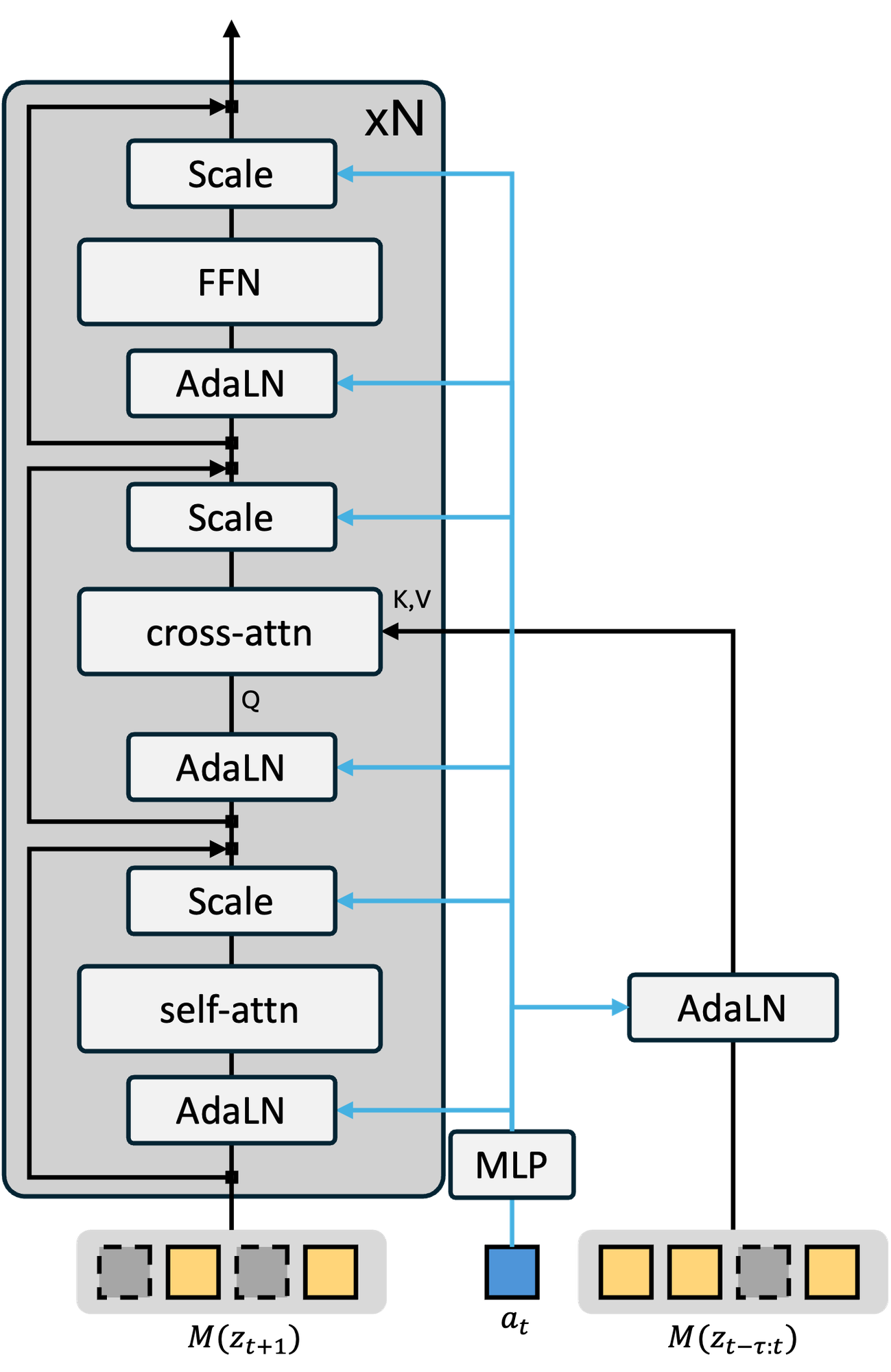}
    \caption{\textbf{World model $f_\phi$ architecture for navigation.} 
For navigation tasks, the world model is formulated as an action-conditioned generative model that autoregressively generates the next frame. 
The model follows the CDiT architecture proposed in NWM~\cite{bar2025navigation}, where actions are conditioned via adaptive layer normalization and history frames are conditioned via cross-attention.}
    \label{fig:wm_ar_arch}
\end{figure}

\begin{figure*}[t!]
    \centering
    \includegraphics[height=4cm]{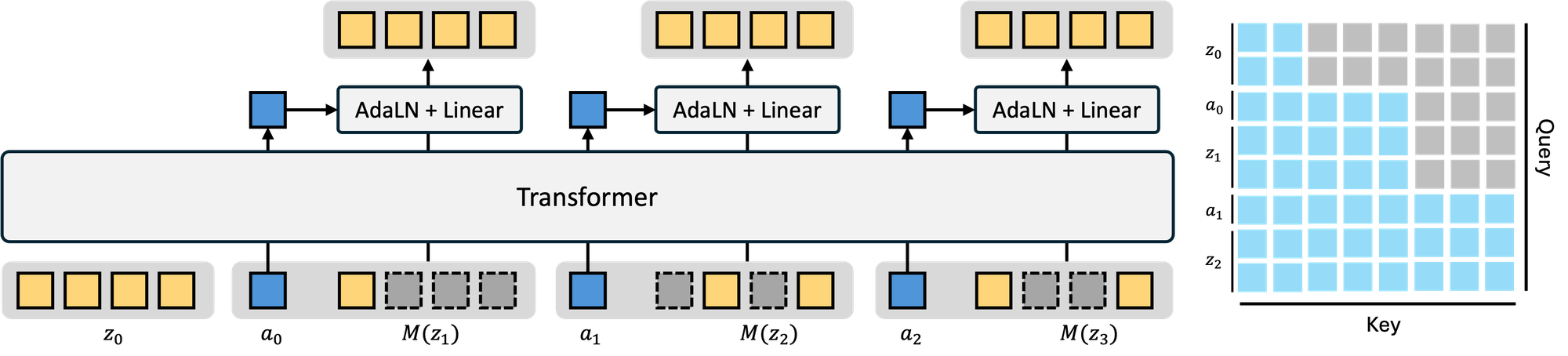}
    \vspace{-3mm}
    \caption{\textbf{World model $f_\phi$ architecture for manipulation.} 
For manipulation tasks, the world model is formulated as a block-causal transformer that enables parallel prediction of multiple future frames conditioned on historical timesteps. 
Action tokens are used to condition the prediction head (AdaLN + linear layer) that unmasks latent tokens at each future timestep.
    }
    \label{fig:wm_block_causal_arch}
\end{figure*}

\section{World model details}
\label{sec:supp_wm_arch}

\begin{table}
\small
\centering
\caption{
\textbf{Model architecture hyperparameters for world model.}
}
\vspace{-3mm}
\scalebox{0.85}{
\begin{tabular}{l|cc}
 \toprule
 hparams & navigation & manipulation \\ \midrule
depth & 12 & 16 \\ 
dim & 768 & 1024 \\
MLP dim & 3072 & 4096 \\
heads & 12 & 16 \\
\# params & 243M & 270M \\
\bottomrule
\end{tabular}
}
\label{tab:wm_nav_model_hparams}
\end{table}

\subsection{Navigation: autoregressive with fixed history window}
\label{sec:supp_wm_arch_nav}
For navigation tasks, we adopt an autoregressive formulation where the world model predicts the next latent state $z_{t+1}$ conditioned on a fixed-length history window of past states $\{z_{t-\tau}, \ldots, z_t\}$ and the action $a_t$. This design follows the Navigation World Models (NWM)~\cite{bar2025navigation} framework but operates in our compact discrete latent space with $N \leq 16$ tokens per timestep.
Fig.~\ref{fig:wm_ar_arch} and Tab.~\ref{tab:wm_nav_model_hparams} present the model architecture and hyperparameters of world model for navigation, respectively.

\noindent \textbf{Architecture.}
Fig.~\ref{fig:wm_ar_arch} illustrates the architecture. We employ a DiT-based~\cite{peebles2023scalable} architecture with multiple transformer blocks, where each block consists of adaptive layer normalization (AdaLN), multi-head self-attention, cross-attention, and feed-forward networks (FFN). The action $a_t$ is first encoded through an MLP and used to modulate the transformer blocks via AdaLN, similar to how DiT conditions on class labels. 
The model takes as input the masked future tokens $M(z_{t+1})$ and attends to the history window $M(z_{t-\tau:t})$ through cross-attention layers, where $M(\cdot)$ denotes the masking operation. The history tokens serve as keys and values (K, V), while the future tokens act as queries (Q). 
We use $\tau = 4$ for the history window length and $N = 12$ for the DiT blocks.
For SD-VAE experiments, we follow the exact architecture from NWM~\cite{bar2025navigation} (CDiT-B). 
For discrete tokenizers (FlexTok~\cite{bachmann2025flextok} and \modelname), we use the same configuration with two adaptations to handle discrete tokens: (1) embedding and classification layers instead of continuous projections, and (2) masked generative modeling instead of diffusion-based generation.

\noindent \textbf{History masking.}
Following the principles of diffusion forcing~\cite{chen2024diffusion}, we randomly mask tokens in the history window during training. This encourages the model to learn robust temporal dependencies and improves action conditioning. During each training iteration, we randomly mask tokens in each historical frame. At inference time, we use the slightly masked (20\%) history for stable rollouts. Ablation results in Table~5 of the main paper demonstrate that history masking improves planning accuracy.

\noindent \textbf{Dataset.}
The navigation world model is trained on three datasets: RECON~\cite{shah2021rapid}, SCAND~\cite{karnan2022socially}, and HuRoN~\cite{hirose2023sacson}. 
For HuRoN, we use the publicly available low-resolution version.
All datasets contain first-person navigation trajectories with corresponding actions (changes in x-axis, y-axis, and yaw). Frames are center-cropped and resized to $224 \times 224$ resolution. We follow the same train/test splits as NWM~\cite{bar2025navigation}. 
We exclude Tartan~\cite{wang2020tartanair} and Ego4D~\cite{grauman2022ego4d}, which were used in the original NWM~\cite{bar2025navigation}, for the following reasons: Tartan consists predominantly of forward-only actions, and Ego4D, while providing large-scale data, is computationally expensive to process. We found that we can fairly reproduce navigation planning performance without them.

\noindent \textbf{Reproducing NWM~\cite{bar2025navigation}.}
The SD-VAE row in Table~4 of the main paper can be considered as our reproduction of NWM~\cite{bar2025navigation}. The original NWM reports ATE of 1.13 and RPE of 0.35 on RECON, while our reproduction achieves ATE of 1.262 and RPE of 0.354. Despite several differences—using a smaller model (CDiT-B instead of CDiT-XL used in the original NWM), low-resolution data for HuRoN, and excluded datasets (Tartan and Ego4D)—we fairly reproduce the navigation planning performance.

\noindent \textbf{Training and inference.}
The navigation world model is trained for 200K steps with a batch size of 128. We use AdamW optimizer~\cite{loshchilov2017decoupled} with a learning rate of $1 \times 10^{-4}$.  We set $\beta_1 = 0.9$ and $\beta_2 = 0.999$, with weight decay of $1 \times 10^{-2}$.
For the masked generative modeling objective, we randomly sample the mask ratio from $(0, 1]$ following the cosine masking schedule~\cite{chang2022maskgit}.
During inference, we used 8 and 4 sampling steps for 16 and 8 latent tokens, respectively.

\noindent \textbf{Evaluation.}
We evaluate navigation planning performance using two standard trajectory metrics: Absolute Trajectory Error (ATE) and Relative Pose Error (RPE)~\cite{sturm2012evaluating}. 
First, ATE measures the RMSE between predicted and ground truth transformations from the initial frame across the entire trajectory between corresponding timesteps. 
Second, RPE measures the error in relative transformations between consecutive timesteps. 
Lower values indicate better planning performance for both metrics.
For details of model-predictive control framework used for navigation planning, we refer Sec.~\ref{sec:supp_planning}.

\subsection{Manipulation: block-causal parallel prediction}
\label{sec:supp_wm_arch_mani}
For robotic manipulation tasks on RoboNet~\cite{dasari2019robonet}, we adopt a block-causal transformer architecture that predicts multiple future frames $\{z_{t+1}, \ldots, z_{t+K}\}$ in parallel, conditioned on the initial observation $z_t$ and action sequence $\{a_t, \ldots, a_{t+K-1}\}$. This parallel formulation is more suitable for video generation tasks where we need to produce extended action-conditioned rollouts efficiently.
Fig.~\ref{fig:wm_block_causal_arch} and Tab.~\ref{tab:wm_nav_model_hparams} present the model architecture and hyperparameters of world model for manipulation, respectively.

\noindent \textbf{Architecture.}
Figure~\ref{fig:wm_block_causal_arch} illustrates the block-causal architecture. Unlike the autoregressive model, all future frames are processed simultaneously within a single transformer. The input sequence is constructed by interleaving observations and actions: $[z_{t-\tau:t}, a_t, M(z_{t+1}), \ldots, a_{t+H-1}, M(z_{t+H})]$, where $M(\cdot)$ denotes masked tokens. 
The key component is the block-causal attention mask (shown on the right of Figure~\ref{fig:wm_block_causal_arch}): tokens at timestep $t+i$ can attend to all tokens up to and including $z_{t+i}$ and $a_{t+i-1}$, but cannot attend to future observations or actions. This preserves causal structure while enabling parallel prediction. 
Each action is encoded through a linear layer. The output token corresponding to each action is then used to condition a prediction head (AdaLN + Linear layer), which produces the unmasked tokens for the subsequent frame. 
During training, we set the prediction horizon to $H=14$ (predicting the next 14 frames in parallel). However, the model can generalize to arbitrary horizon lengths since we use causal masking. We provide the first two frames of each episode as context for the model to condition on ($\tau=2$).

\noindent \textbf{Baseline configuration.}
For fair comparison, experiments using the target tokenizer (MaskGIT-VQGAN~\cite{chang2022maskgit}) employ the same block-causal architecture. The only difference is the sequence length: the target tokenizer processes 256 tokens per frame, while CompACT processes only 16 tokens per frame, resulting in significantly faster generation (Tab.~6 in the main paper).

\noindent \textbf{Diffusion forcing interpretation.}
This architecture naturally implements causal, discrete extension of diffusion forcing~\cite{chen2024diffusion}: during training, masked tokens at different timesteps provide varying levels of noisy conditioning to future frames. A frame with fewer masked tokens acts as a cleaner conditioning signal, while heavily masked frames force the model to rely more on learned dynamics and action conditioning. This training scheme improves the model's ability to generate consistent action-conditioned video sequences.

\noindent \textbf{Dataset.}
The manipulation world model is trained on RoboNet~\cite{dasari2019robonet}, similar to IDM experiment.
Frames are pre-encoded using robonet finetuned \modelname with 256 $\times$ 256 resolution center-cropped images. 
{For the target tokenizer~\cite{chang2022maskgit} baseline, we use it as-is without RoboNet finetuning, which is fair since finetuned \modelname also keeps the frozen target tokenizer.}
250 episodes are used for test.

\noindent \textbf{Training and inference.}
The manipulation world model is trained for 100K steps with a batch size of 128. We use AdamW optimizer~\cite{loshchilov2017decoupled} with a learning rate of $1 \times 10^{-4}$. We set $\beta_1 = 0.9$ and $\beta_2 = 0.999$, with weight decay of $1 \times 10^{-2}$. 
For the masked generative modeling objective, we randomly sample the mask ratio from $(0, 1]$ following the cosine masking schedule~\cite{chang2022maskgit}.
We used 100 sampling steps to generate future 14 frames.

\begin{algorithm*}[h]
\caption{Cross-Entropy Method for Navigation Planning}
\begin{algorithmic}[1]
\label{alg:cem}
\REQUIRE Initial frame $o_0$, goal frame $o_{\text{goal}}$, planning horizon $H$, World model $f_\phi$, tokenizer $\mathcal{D}_{\text{compact}} \circ \mathcal{E}_{\text{compact}}$
\REQUIRE CEM params: population size $N$, top samples to be selected $K$, iterations $I$, samples per candidate $M$
\STATE \textcolor{gray}{// Initialize action distribution}
\STATE $\mu \gets (\mu_x, \mu_y, \mu_\phi)$ 
\STATE $\Sigma \gets \text{diag}(\sigma_x^2, \sigma_y^2, \sigma_\phi^2)$ 
\STATE Encode observations: $z_0 \gets \mathcal{E}_{\text{compact}}(o_0)$, $z_{\text{goal}} \gets \mathcal{E}_{\text{compact}}(o_{\text{goal}})$
\FOR{$i = 1$ to $I$}
    \STATE $\mathcal{A} \gets \{(u_j, \phi_j)\}_{j=1}^N$ where $(u_j, \phi_j) \sim \mathcal{N}(\mu, \Sigma)$ \textcolor{gray}{// Sample candidate actions}
    \FOR{each candidate $(u_j, \phi_j) \in \mathcal{A}$}
        \STATE \textcolor{gray}{// Evaluate via stochastic rollouts}
        \FOR{$m = 1$ to $M$}
            \STATE $z_t^{(m)} \gets z_0$
            \STATE $\mathbf{a}^{(j)} \gets \mathtt{ToActionSeq}(u_j, \phi_j, H)$
            \FOR{$t = 0$ to $H-1$}
                \STATE $z_{t+1}^{(m)} \sim f_\phi(z_t^{(m)}, a_t^{(j)})$ \textcolor{gray}{// World model rollout}
            \ENDFOR
            \STATE $\hat{o}_H^{(m)} \gets \mathcal{D}_{\psi} \circ \mathcal{D}_{\text{compact}}(z_H^{(m)})$
            \STATE $\hat{o}_\text{goal} \gets \mathcal{D}_{\psi} \circ \mathcal{D}_{\text{compact}}(z_\text{goal})$
            \STATE $c_m \gets d(\hat{o}_H^{(m)}, \hat{o}_{\text{goal}})$ or $c_m \gets d(z_H^{(m)}, z_{\text{goal}})$\textcolor{gray}{// LPIPS between reconstruction or L1 distance between latent}
        \ENDFOR
        \STATE $C_j \gets \frac{1}{M}\sum_{m=1}^M c_m$
    \ENDFOR
    \STATE $\tilde{\mathcal{A}} \gets$ top-$K$ from $\mathcal{A}$ with lowest costs $\{C_j\}$ \textcolor{gray}{// Select top-K candidates}
    \STATE \textcolor{gray}{// Update distribution parameters}
    \STATE $\mu \gets \frac{1}{K}\sum_{(u,\phi) \in \tilde{\mathcal{A}}} (u, \phi)$
    \STATE $\Sigma \gets \frac{1}{K}\sum_{(u,\phi) \in \tilde{\mathcal{A}}} ((u,\phi) - \mu)((u,\phi) - \mu)^\top$
\ENDFOR
\STATE \textbf{return} Action with minimum cost from $\tilde{\mathcal{A}}$
\end{algorithmic}
\end{algorithm*}

\section{Planning with cross-entropy method}
\label{sec:supp_planning}
As described in Sec.~3.1 and Sec.~3.3, we use a trained world model to standalone-plan goal-conditioned navigation trajectories by optimizing distance between final prediction and the goal. 
Here, we provide additional details about the optimization using the Cross-Entropy Method (CEM)~\cite{de2005cem,chua2018cem} and the hyperparameters used.

\noindent\textbf{Cross-Entropy Method.}
For CEM planning, we strictly follow the protocol of NWM~\cite{bar2025navigation} for fair comparison. Algo.~\ref{alg:cem} presents the complete CEM procedure.
Specifically, trajectory is assumed to be a straight line and only its endpoint is optimized, represented by three variables: a single translation $u$ and yaw rotation $\phi$.
This is converted into an action sequence by evenly distributing the translation across $H$ timesteps and applying the rotation only at the final step.
Each action step corresponds to a fixed time interval of 0.25 seconds. 

\noindent\textbf{Hyperparameters.}
We set the population size to 80 since increasing over it does not lead to substantial gain in planning accuracy.
For other hyperparameters, we follow the configuration of NWM: 2 seconds planning ($H=8$), single iteration ($I=1$), top-$K$ selection with $K=5$, and $M=3$ repetitive stochastic rollouts per candidate action.

\noindent\textbf{Distance metric $d(\cdot, \cdot)$.}
The distance between predicted and goal observations can be measured in either pixel space or latent space.
For pixel-space evaluation, we decode latent tokens to images using $\mathcal{D}_{\psi} \circ \mathcal{D}_{\text{compact}}$ and compute the LPIPS distance~\cite{johnson2016perceptual}.
Alternatively, we can compute distances directly in the discrete latent space, which offers significant computational speedup by avoiding the decoding step (Tab.~5 (\textit{middle}) in the main paper).
This latent-space distance computation is enabled by the properties of Finite Scalar Quantization (FSQ)~\cite{mentzer2023finite}, which we use for discretization.
Unlike traditional vector quantization that uses arbitrary codebook indices, FSQ assigns each latent vector to a discrete code based on level-based radix representation, preserving the continuous geometric structure in the discrete space.
Specifically, each dimension of the latent is quantized to one of $L$ equally-spaced levels, allowing us to map discrete token indices back to their level-based radix representation.
Given two discrete latent tokens $z_i, z_j \in \{1, \ldots, K\}$, we compute their distance as:
\begin{equation}
d(z_i, z_j) = \|\mathtt{FSQ}^{-1}(z_i) - \mathtt{FSQ}^{-1}(z_j)\|_1,
\end{equation}
where $\mathtt{FSQ}^{-1}(\cdot)$ maps the discrete index to its corresponding level-based radix representation.
This distance in the discrete latent space provides a meaningful measure of semantic similarity while enabling faster planning with marginal degradation compared to pixel-space metrics.

\begin{figure}[t!]
    \centering
    \includegraphics[width=0.45\textwidth]{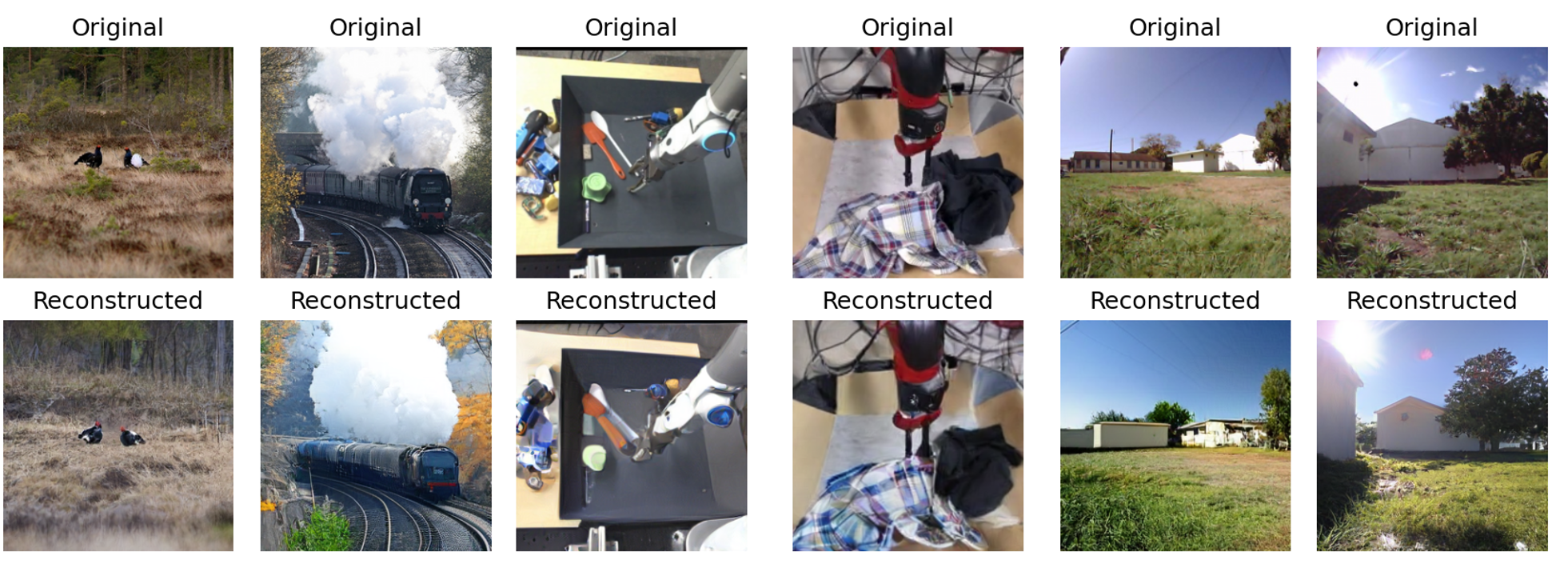}
    \vspace{-2mm}
    \caption{\textbf{Qualitative results of reconstruction with \modelname.}}
    \vspace{-4mm}
    \label{fig:qual_tok_recon}
\end{figure}

\begin{figure}[t!]
    \centering
    \includegraphics[width=0.45\textwidth]{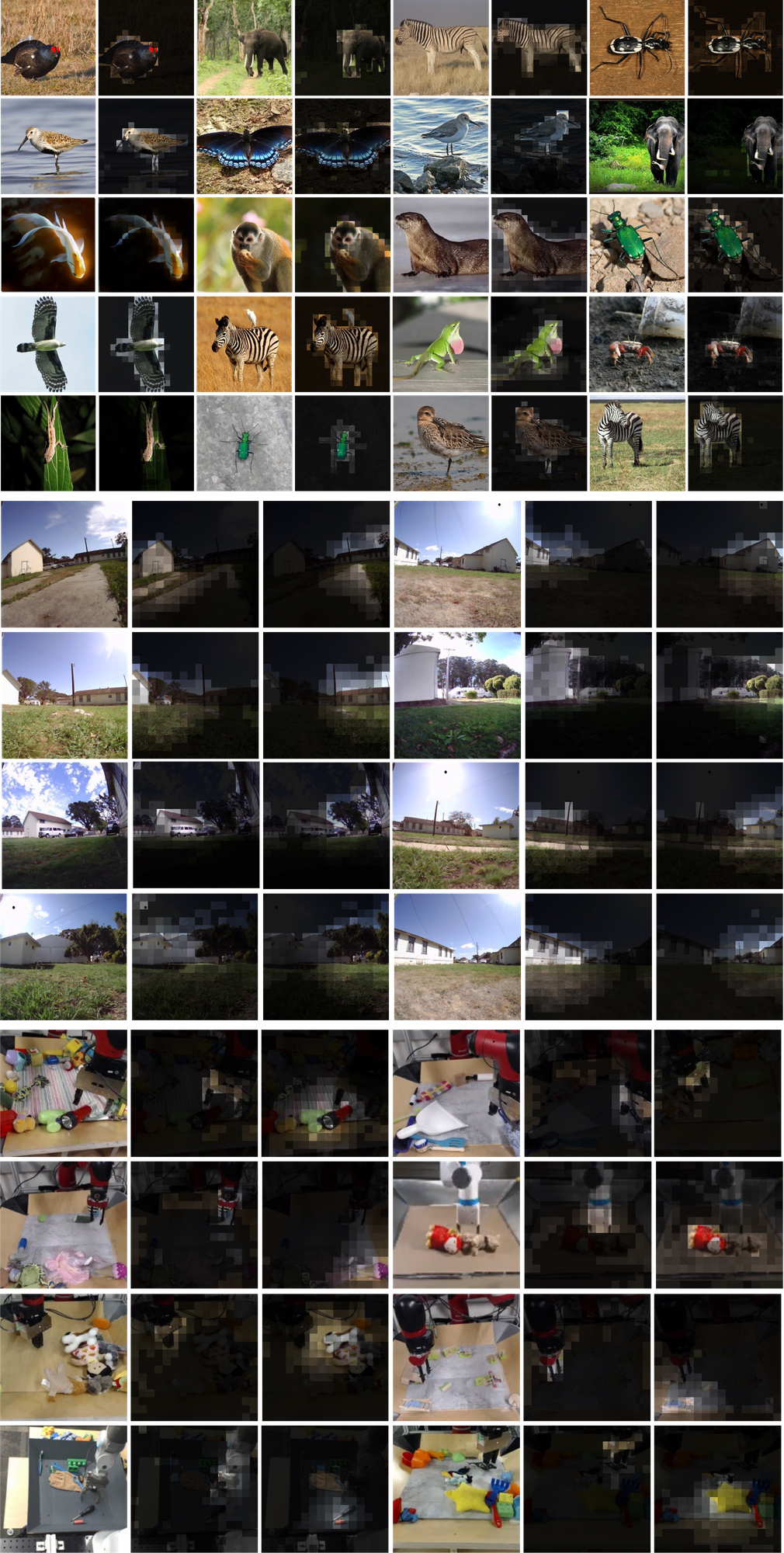}
    \vspace{-2mm}
    \caption{\textbf{Attention visualization for compact latent tokens in latent resampler.} Brighter the color, higher the attention score.}
    \vspace{-2mm}
    \label{fig:qual_tok_attn}
\end{figure}

\section{More qualitative results}
\label{sec:supp_qual}
Fig.~\ref{fig:qual_tok_recon} shows reconstruction examples across ImageNet, RECON, and RoboNet. While CompACT discards fine-grained textures and lighting details due to extreme compression, it preserves semantic content and spatial structure essential for planning tasks. 
Fig.~\ref{fig:qual_tok_attn} visualizes attention patterns in the latent resampler, demonstrating that each compact token attends to semantically coherent regions.
Fig.~\ref{fig:qual_plan_supp} presents example planning results with CompACT. While fine-grained details such as textures and shadows are synthesized rather than reconstructed, the rollouts accurately preserve planning-critical information: spatial layout, object positions, and scene structure necessary for effective goal-reaching.
Fig.~\ref{fig:qual_robonet} presents additional examples of action-conditioned video generation on RoboNet. Videos generated from CompACT latents maintain more consistent action-driven end-effector movements throughout the rollout compared to the ones generated with target tokenizer latents, validating that the modular latent tokens effectively capture dynamics-relevant information for manipulation tasks.

\section{Planning efficiency analysis}
\label{sec:supp_eff}
Fig.~\ref{fig:2d_plot} visualizes the trade-off between planning accuracy (ATE), planning latency, and model size across different tokenizers on RECON. Bubble size represents the peak VRAM usage during planning. 
CompACT variants achieve superior efficiency: CompACT delivers up to 80× speedup over SD-VAE while maintaining comparable accuracy. In contrast, FlexTok variants suffer significant accuracy loss and large VRAM requirements despite similar token counts, validating that our tokenizer design is critical for enabling efficient real-time planning.

\begin{figure}[h!]
    \centering
    \includegraphics[width=0.45\textwidth]{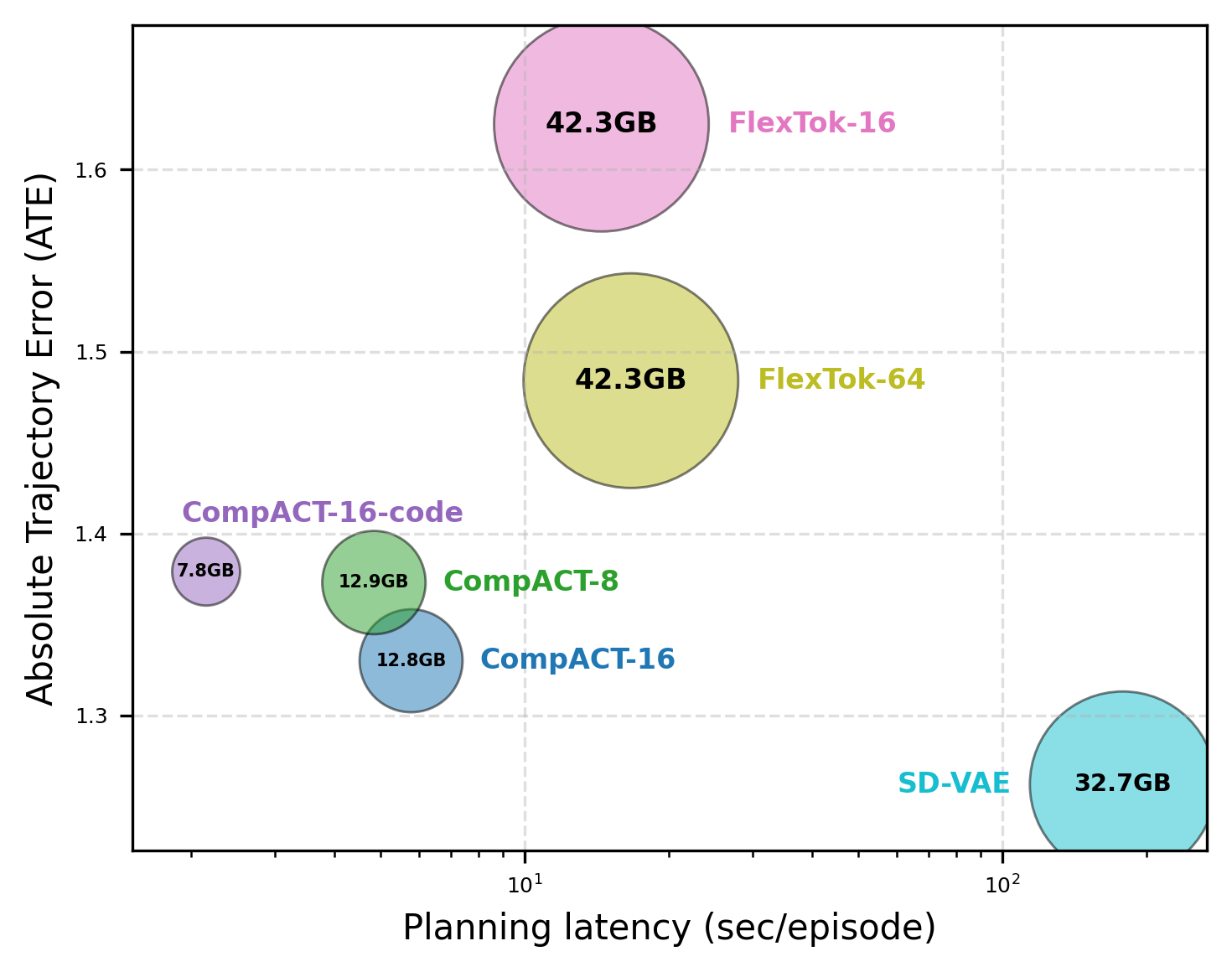}
    \vspace{-2mm}
    \caption{\textbf{Plot for ATE, planning latency, and memory peak usage on RECON~\cite{shah2021rapid}.} Latency and memory usage is mesaured for single trajectory optimization, using a single RTX 6000 ADA GPU.}
    \vspace{-4mm}
    \label{fig:2d_plot}
\end{figure}

\section{Scaling up with fewer tokens}
\label{sec:supp_scale_up}
The compact latent representation of CompACT enables scaling world models to larger capacities while maintaining practical planning latency. 
We train a 750M-parameter variant of our world model for navigation by increasing the depth to 24 layers and hidden dimension to 1024. 
This scaled model achieves improved planning accuracy with ATE of 1.305 and RPE of 0.370 on RECON---outperforming our base 16-token model (ATE=1.330, RPE=0.390). 
Planning latency is 24.7 seconds per trajectory, still 7$\times$ faster than the SD-VAE baseline (178.78 seconds). 

\clearpage
\begin{figure*}[t!]
    \centering
    \includegraphics[width=0.96\textwidth]{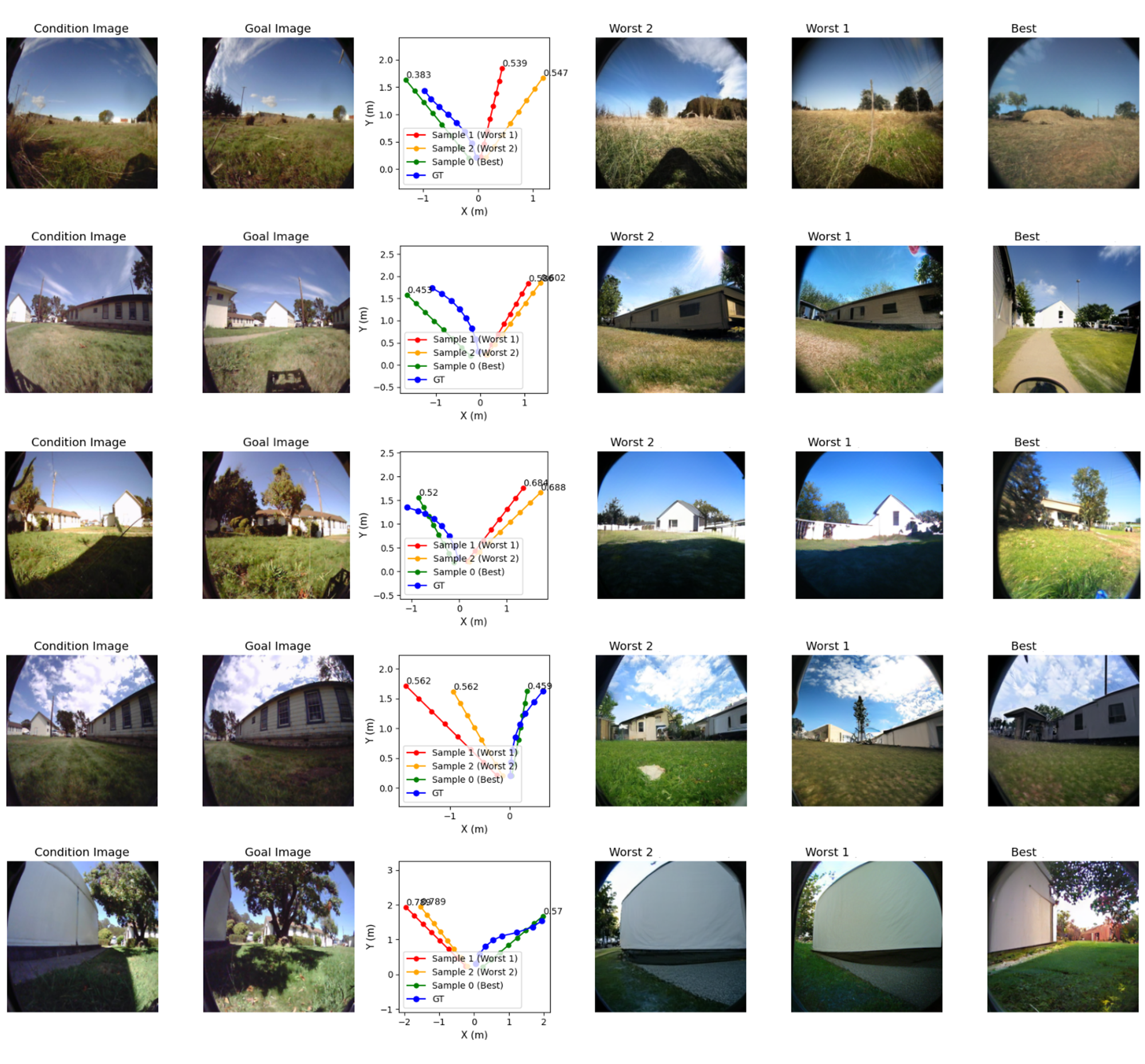}
    \caption{\textbf{Additional qualitative results of navigation planning with the proposed CompACT.} Among candidates, worst two action sequences and best action sequences are presented together.}
    \label{fig:qual_plan_supp}
\end{figure*}

\clearpage
\begin{figure*}[t!]
    \centering
    \includegraphics[height=0.96\textheight]{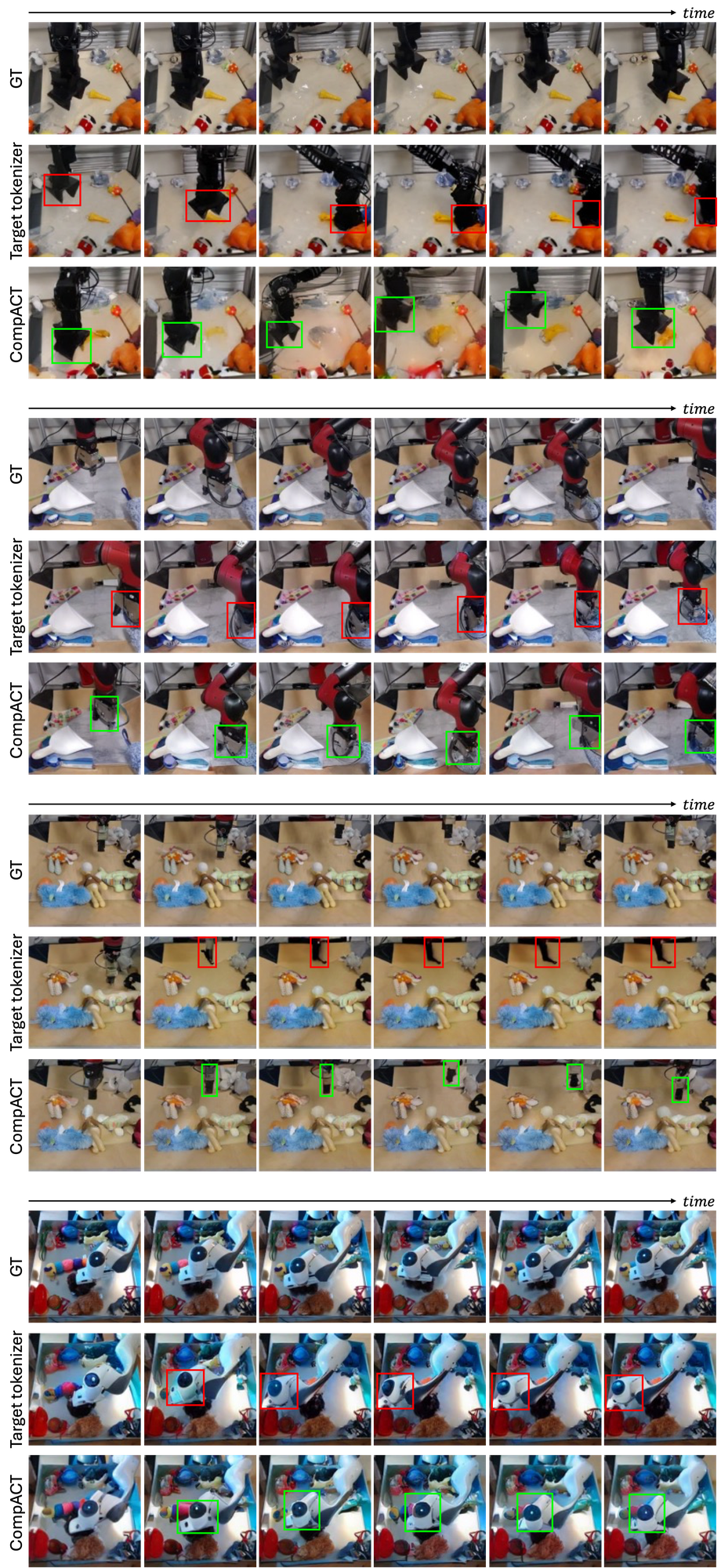}
    \caption{\textbf{Additional qualitative results of action-conditioned video generation.} Red and green boxes indicate incorrect and correct end-effector positions, respectively.}
    \label{fig:qual_robonet}
\end{figure*}

%% file: sec/supp/planning_suff.tex
\section{Planning sufficiency of compact tokens}
\label{sec:supp_planning_sufficiency}

How many bits must a latent representation $\vz$ carry so that planning in latent space is equivalent to planning in observation space?
We formalize this question using mutual information and derive a tight lower bound on the entropy of any planning-sufficient representation $\vz$.
Let $\vo$ denote an observation, $\va^{*}$ denote the optimal action (or action sequence), and $\vz = \mathcal{E}(\vo)$ denote the latent representation.

\begin{definition}[Planning sufficiency]
\label{def:planning_sufficiency}
A latent representation $\vz = \mathcal{E}(\vo)$ is \emph{planning-sufficient} if, for a optimal planning algorithm $\pi$, planning in latent space yields the same optimal action as planning in observation space: $\pi(\vz) = \pi(\vo)$.
This requires that $\vz$ retains all action-relevant information from $\vo$, i.e.,
\begin{equation}
I(\vz;\, \va^{*}) \;=\; I(\vo;\, \va^{*}).
\label{eq:planning_suff}
\end{equation}
\end{definition}

\begin{proposition}[Minimum description length for planning]
\label{prop:mdl_bound}
If the optimal planning algorithm $\pi$ is deterministic, i.e., $H(\va^{*}\mid\vo)=0$,
then a planning-sufficient representation $\vz$ 
(Def.~\ref{def:planning_sufficiency}) exists with minimum entropy
\begin{equation}
  \min_{\vz:\,\text{plan-suff.}} H(\vz) 
  \;=\; I(\vo;\,\va^{*}) \;=\; H(\va^{*}) 
  \;\ll\; H(\vo),
  \label{eq:mdl_tight}
\end{equation}
established by necessity (no planning-sufficient $\vz$ can have 
lower entropy) and achievability (a $\vz$ attaining this bound exists).
\end{proposition}

\begin{proof}
\leavevmode
\begin{enumerate}[label=\textbf{(\alph*)},leftmargin=*,nosep]
\item \textbf{(Necessity)}\;
By the definition of mutual information, $I(\vz;\, \va^{*}) = H(\vz) - H(\vz \mid \va^{*}) \leq H(\vz)$.
Combining this with the planning-sufficiency condition $I(\vz;\,\va^{*}) = I(\vo;\,\va^{*})$ (Eq.~\ref{eq:planning_suff}) yields $H(\vz) \geq I(\vo;\, \va^{*})$.
\item \textbf{(Achievability)}\;
When $H(\va^{*}\mid\vo)=0$, define the deterministic encoder $\mathcal{E}(\vo) := \pi^{*}(\vo) = \va^*$.
Then $I(\vz;\,\va^{*}) = H(\va^{*}) = I(\vo;\,\va^{*})$, so $\vz$ is planning-sufficient, 
and $H(\vz) = H(\va^{*}) = I(\vo;\,\va^{*})$, achieving the bound. \qedhere
\end{enumerate}
\end{proof}

\noindent \textbf{Remark.}
The deterministic assumption $H(\va^{*}\mid\vo)=0$ refers to the existence of a unique optimal action for each observation, not to the planning procedure itself.
In practice, \modelname uses the Cross-Entropy Method (CEM), which employs stochastic sampling to search for this optimum; the stochasticity is in the optimizer, not in the underlying observation-to-action mapping.

\paragraph{Connection to \modelname.}
In robotic manipulation and navigation, action spaces are low-dimensional (e.g., 3--5 DOF), so $H(\va^{*}) \ll H(\vo)$.
Proposition~\ref{prop:mdl_bound} implies that a planning-sufficient tokenizer need not preserve the full observation entropy as in conventional autoencoder; the information-theoretic floor is at most $H(\va^{*})$ bits.
\modelname encodes each frame into 8--16 discrete tokens using FSQ with levels $[8,8,8,5,5,5]$, yielding $\approx\! 2^{16}$ codes per token ($\approx\! 16$ bits/token), for a total of 128--256 bits per frame.
Considering that previous work leveraged $\approx\!$ 300--400 bits to encode the whole image~\cite{yu2024image,kim2025democratizing} and $H(\va^{*}) \ll H(\vo)$, this is expected to far exceed the action-entropy bound, providing a comfortable margin for planning sufficiency.
Empirically, Tab.~\ref{tab:idm} confirms this: an inverse dynamics model trained on \modelname tokens (16 tokens, $R^2 = 0.716$) outperforms one trained on the baseline tokenizer using $16\times$ more tokens (256 tokens, $R^2 = 0.684$), indicating that the compact representation retains and even enhances the action-relevant information.

%% file: sec/supp/other_backbone.tex
\section{Cross-backbone ablation for \modelname}
\label{sec:supp_backbone}
\begin{table}[h]
\centering
\caption{\textbf{Cross-backbone ablation of \modelname.} Reconstruction performance of \modelname with 16 tokens on ImageNet validation split across different backbone choices.}
\resizebox{0.7\columnwidth}{!}{
\begin{tabular}{lccc}
\toprule
Backbone & SigLIP-2 & MAE & DINOv3\\
\midrule
rFID$\downarrow$ & \textbf{2.09} & 3.43 & 2.40 \\
\bottomrule
\end{tabular}
}
\label{tab:cross_backbone}
\end{table}
\noindent The main paper presents \modelname with DINOv3~\cite{simeoni2025dinov3} as the frozen encoder backbone.
To verify that our approach does not depend on properties specific to DINOv3, we replace it with two architecturally distinct vision foundation models (VFMs)---MAE~\cite{he2022masked}, trained via masked image modeling, and SigLIP-2~\cite{tschannen2025siglip}, trained via vision-language pretraining---while keeping all other components unchanged.
As shown in Tab.~\ref{tab:cross_backbone}, all three backbones yield competitive rFID, with SigLIP-2 even surpassing DINOv3.
This result confirms that \modelname is not tied to any particular backbone's latent space; rather, it leverages the high-level semantic structure shared across strong VFMs---the planning-relevant information our tokenizer is designed to preserve.

%% file: sec/supp/latency_breakdown.tex
\section{Planning latency breakdown}
\label{sec:supp_latency_breakdown}

\begin{figure}[t]
    \centering
    \includegraphics[width=0.95\columnwidth]{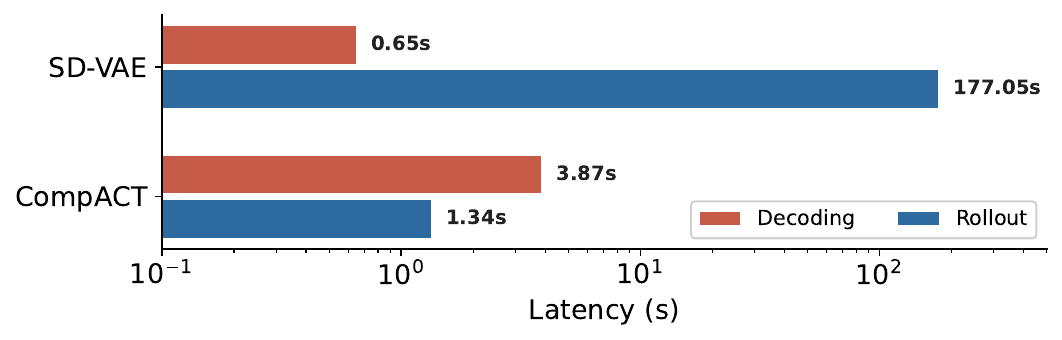}
    \vspace{-2mm}
    \caption{\textbf{Planning latency breakdown.} Comparison of rollout and decoding latency between SD-VAE and \modelname on a single RTX 6000 ADA GPU. Note the log scale on the $x$-axis.}
    \label{fig:latency_breakdown}
\end{figure}

Fig.~\ref{fig:latency_breakdown} two major components of the planning latency for a single trajectory optimization into two components: world model \textit{rollout} (forward passes through the world model) and \textit{decoding} (converting latent tokens back to images for cost evaluation).
The rollout cost dominates SD-VAE's total planning time (177.05s out of 177.70s), as the world model must process 784 tokens per frame across 1,920 forward passes during MPC.
\modelname reduces this bottleneck by 99.2\% (177.05s $\rightarrow$ 1.34s) by operating on only 16 discrete tokens per frame.
While our generative decoder is $\sim$6$\times$ slower than SD-VAE's single-step decoder (3.87s vs.\ 0.65s), decoding accounts for only a small fraction of the total planning budget and is invoked far fewer times than rollout (240 vs.\ 1,920 calls).
Furthermore, when using latent-space cost functions (Tab.~5 in the main paper), decoding is bypassed entirely, yielding even greater speedups.

%% file: sec/supp/robomimic_closed_loop.tex
\section{Closed-loop robot arm manipulation}
\label{sec:supp_rm_closed_loop}
\begin{table}[t]
\centering
\small
\caption{\textbf{Closed-loop manipulation on RoboMimic Lift.}}
\label{tab:robomimic}
\resizebox{0.8\columnwidth}{!}{
\begin{tabular}{lcc}
\toprule
Tokenizer & SR $\uparrow$ & Steps $\downarrow$ \\
\midrule
Target tokenizer ($\mathcal{D}_{\psi} \circ \mathcal{E}_{\psi}$)~\cite{chang2022maskgit} & 56\% & 66.8 \\
CompACT (ours) & \textbf{56\%} & \textbf{55.1} \\
\bottomrule
\end{tabular}
}
\end{table}

To evaluate CompACT beyond open-loop video prediction, we conduct a preliminary closed-loop manipulation experiment on the Lift task from RoboMimic~\cite{robomimic2021}. 
We adopt a hierarchical formulation: the world model plans in CompACT's latent space, and the IDM decodes inter-frame actions for execution. 
For detailed implementation, we refer to Sec.~\ref{sec:supp_wm_arch_mani} and Sec.~\ref{sec:supp_idm}, respectively.

As shown in Tab.~\ref{tab:robomimic}, CompACT (16 tokens) matches the target tokenizer (256 tokens) in success rate while requiring 17\% fewer action steps, suggesting that semantic-level latent plans yield more direct paths to the goal. 
This result provides preliminary evidence that planning-sufficient compression extends to contact-rich manipulation tasks. 
We note this is a minimal setup; incorporating real-time observations and proprioception into the IDM could further improve performance, which we leave for future work.

%% file: main.bib
@string{AAAI = {Proc. AAAI Conference on Artificial Intelligence (AAAI)}}

@string{CVPR = {Proc. IEEE Conference on Computer Vision and Pattern Recognition (CVPR)}}

@string{ECCV = {Proc. European Conference on Computer Vision (ECCV)}}

@string{ICCV = {Proc. IEEE International Conference on Computer Vision (ICCV)}}

@string{ICLR = {Proc. International Conference on Learning Representations (ICLR)}}

@string{ICML = {Proc. International Conference on Machine Learning (ICML)}}

@string{NIPS = {Proc. Neural Information Processing Systems (NeurIPS)}}

@string{IROS = {Proc. IEEE/RSJ International Conference on Intelligent Robots and Systems (IROS)}}

@string{CoRL = {Proc. Annual Conference on Robot Learning (CoRL)}}

@string{RSS = {Robotics: Science and Systems (RSS)}}

@string{TMLR = {Transactions on Machine Learning Research (TMLR)}}

@inproceedings{li2023blip,
  title={Blip-2: Bootstrapping language-image pre-training with frozen image encoders and large language models},
  author={Li, Junnan and Li, Dongxu and Savarese, Silvio and Hoi, Steven},
  booktitle=ICML,
  year={2023},
}

@inproceedings{locatello2020object,
  title={Object-centric learning with slot attention},
  author={Locatello, Francesco and Weissenborn, Dirk and Unterthiner, Thomas and Mahendran, Aravindh and Heigold, Georg and Uszkoreit, Jakob and Dosovitskiy, Alexey and Kipf, Thomas},
  booktitle=NIPS,
  year={2020}
}

@inproceedings{ryoo2021tokenlearner,
  title={Tokenlearner: What can 8 learned tokens do for images and videos?},
  author={Ryoo, Michael S and Piergiovanni, AJ and Arnab, Anurag and Dehghani, Mostafa and Angelova, Anelia},
  booktitle=NIPS,
  year={2021}
}

@article{gao2025one,
  title={One Layer Is Enough: Adapting Pretrained Visual Encoders for Image Generation},
  author={Gao, Yuan and Chen, Chen and Chen, Tianrong and Gu, Jiatao},
  journal={arXiv preprint arXiv:2512.07829},
  year={2025}
}

@inproceedings{zheng2025diffusion,
  title={Diffusion transformers with representation autoencoders},
  author={Zheng, Boyang and Ma, Nanye and Tong, Shengbang and Xie, Saining},
  booktitle=ICLR,
  year={2026}
}

@inproceedings{wu2024ivideogpt,
    title={iVideoGPT: Interactive VideoGPTs are Scalable World Models}, 
    author={Jialong Wu and Shaofeng Yin and Ningya Feng and Xu He and Dong Li and Jianye Hao and Mingsheng Long},
    booktitle=NIPS,
    year={2024},
}

@article{tschannen2025siglip,
  title={SigLIP 2: Multilingual Vision-Language Encoders with Improved Semantic Understanding, Localization, and Dense Features},
  author={Tschannen, Michael and Gritsenko, Alexey and Wang, Xiao and Naeem, Muhammad Ferjad and Alabdulmohsin, Ibrahim and Parthasarathy, Nikhil and Evans, Talfan and Beyer, Lucas and Xia, Ye and Mustafa, Basil and H\'enaff, Olivier and Harmsen, Jeremiah and Steiner, Andreas and Zhai, Xiaohua},
  year={2025},
  journal={arXiv preprint arXiv:2502.14786}
}

@inproceedings{robomimic2021,
  title={What Matters in Learning from Offline Human Demonstrations for Robot Manipulation},
  author={Ajay Mandlekar and Danfei Xu and Josiah Wong and Soroush Nasiriany and Chen Wang and Rohun Kulkarni and Li Fei-Fei and Silvio Savarese and Yuke Zhu and Roberto Mart\'{i}n-Mart\'{i}n},
  booktitle=CoRL,
  year={2021}
}

@inproceedings{sturm2012evaluating,
  title={Evaluating egomotion and structure-from-motion approaches using the TUM RGB-D benchmark},
  author={Sturm, J{\"u}rgen and Burgard, Wolfram and Cremers, Daniel},
  booktitle={Proc. of the Workshop on Color-Depth Camera Fusion in Robotics at the IEEE/RJS International Conference on Intelligent Robot Systems (IROS)},
  volume={13},
  pages={6},
  year={2012}
}

@inproceedings{grauman2022ego4d,
  title={Ego4d: Around the world in 3,000 hours of egocentric video},
  author={Grauman, Kristen and Westbury, Andrew and Byrne, Eugene and Chavis, Zachary and Furnari, Antonino and Girdhar, Rohit and Hamburger, Jackson and Jiang, Hao and Liu, Miao and Liu, Xingyu and others},
  booktitle=CVPR,
  pages={18995--19012},
  year={2022}
}

@inproceedings{wang2020tartanair,
  title={Tartanair: A dataset to push the limits of visual slam},
  author={Wang, Wenshan and Zhu, Delong and Wang, Xiangwei and Hu, Yaoyu and Qiu, Yuheng and Wang, Chen and Hu, Yafei and Kapoor, Ashish and Scherer, Sebastian},
  booktitle={IEEE/RSJ International Conference on Intelligent Robots and Systems (IROS)},
  pages={4909--4916},
  year={2020},
  organization={IEEE}
}

@inproceedings{ho2020denoising,
  title={Denoising diffusion probabilistic models},
  author={Ho, Jonathan and Jain, Ajay and Abbeel, Pieter},
  booktitle=NIPS,
  volume={33},
  pages={6840--6851},
  year={2020}
}

@inproceedings{nichol2021improved,
  title={Improved denoising diffusion probabilistic models},
  author={Nichol, Alexander Quinn and Dhariwal, Prafulla},
  booktitle=ICML,
  pages={8162--8171},
  year={2021}
}

@inproceedings{chi2023diffusionpolicy,
	title={Diffusion Policy: Visuomotor Policy Learning via Action Diffusion},
	author={Chi, Cheng and Feng, Siyuan and Du, Yilun and Xu, Zhenjia and Cousineau, Eric and Burchfiel, Benjamin and Song, Shuran},
	booktitle={Proceedings of Robotics: Science and Systems (RSS)},
	year={2023}
}

@inproceedings{perez2018film,
  title={Film: Visual reasoning with a general conditioning layer},
  author={Perez, Ethan and Strub, Florian and De Vries, Harm and Dumoulin, Vincent and Courville, Aaron},
  booktitle=AAAI,
  volume={32},
  number={1},
  year={2018}
}

@article{hendrycks2016gaussian,
  title={Gaussian Error Linear Units (Gelus)},
  author={Hendrycks, D},
  journal={arXiv preprint arXiv:1606.08415},
  year={2016}
}

@inproceedings{loshchilov2017decoupled,
  title={Decoupled weight decay regularization},
  author={Loshchilov, Ilya and Hutter, Frank},
  booktitle=ICLR,
  year={2019}
}

@inproceedings{jaegle2021perceiver,
  title={Perceiver: General perception with iterative attention},
  author={Jaegle, Andrew and Gimeno, Felix and Brock, Andy and Vinyals, Oriol and Zisserman, Andrew and Carreira, Joao},
  booktitle=ICML,
  pages={4651--4664},
  year={2021}
}

@inproceedings{carion2020end,
  title={End-to-end object detection with transformers},
  author={Carion, Nicolas and Massa, Francisco and Synnaeve, Gabriel and Usunier, Nicolas and Kirillov, Alexander and Zagoruyko, Sergey},
  booktitle=ECCV,
  pages={213--229},
  year={2020},
  organization={Springer}
}

@inproceedings{chen2024diffusion,
  title={Diffusion forcing: Next-token prediction meets full-sequence diffusion},
  author={Chen, Boyuan and Mart{\'\i} Mons{\'o}, Diego and Du, Yilun and Simchowitz, Max and Tedrake, Russ and Sitzmann, Vincent},
  booktitle=NIPS,
  volume={37},
  pages={24081--24125},
  year={2024}
}

@inproceedings{kim2025democratizing,
  title={Democratizing text-to-image masked generative models with compact text-aware one-dimensional tokens},
  author={Kim, Dongwon and He, Ju and Yu, Qihang and Yang, Chenglin and Shen, Xiaohui and Kwak, Suha and Chen, Liang-Chieh},
  booktitle=ICCV,
  year={2025}
}

@article{hirose2023sacson,
  title={Sacson: Scalable autonomous control for social navigation},
  author={Hirose, Noriaki and Shah, Dhruv and Sridhar, Ajay and Levine, Sergey},
  journal={IEEE Robotics and Automation Letters},
  volume={9},
  number={1},
  pages={49--56},
  year={2023},
  publisher={IEEE}
}

@article{karnan2022socially,
  title={Socially compliant navigation dataset (scand): A large-scale dataset of demonstrations for social navigation},
  author={Karnan, Haresh and Nair, Anirudh and Xiao, Xuesu and Warnell, Garrett and Pirk, S{\"o}ren and Toshev, Alexander and Hart, Justin and Biswas, Joydeep and Stone, Peter},
  journal={IEEE Robotics and Automation Letters},
  volume={7},
  number={4},
  pages={11807--11814},
  year={2022},
  publisher={IEEE}
}

@inproceedings{zhu2024irasim,
  title={Irasim: Learning interactive real-robot action simulators},
  author={Zhu, Fangqi and Wu, Hongtao and Guo, Song and Liu, Yuxiao and Cheang, Chilam and Kong, Tao},
  booktitle=ICCV,
  year=2025
}

@inproceedings{du2023video,
  title={Video Language Planning},
  author={Du, Yilun and Yang, Mengjiao and Florence, Pete and Xia, Fei and Wahid, Ayzaan and Ichter, Brian and Sermanet, Pierre and Yu, Tianhe and Abbeel, Pieter and Tenenbaum, Joshua B and others},
  booktitle=ICLR,
  year={2024}
}

@inproceedings{ko2023learning,
  title={Learning to act from actionless videos through dense correspondences},
  author={Ko, Po-Chen and Mao, Jiayuan and Du, Yilun and Sun, Shao-Hua and Tenenbaum, Joshua B},
  booktitle=ICLR,
  year={2024}
}

@article{hansen2023td,
  title={Td-mpc2: Scalable, robust world models for continuous control},
  author={Hansen, Nicklas and Su, Hao and Wang, Xiaolong},
  journal={arXiv preprint arXiv:2310.16828},
  year={2023}
}

@inproceedings{micheli2022transformers,
  title={Transformers are sample-efficient world models},
  author={Micheli, Vincent and Alonso, Eloi and Fleuret, Fran{\c{c}}ois},
  booktitle=ICLR,
  year={2023}
}

@inproceedings{hafner2020mastering,
  title={Mastering atari with discrete world models},
  author={Hafner, Danijar and Lillicrap, Timothy and Norouzi, Mohammad and Ba, Jimmy},
  booktitle=ICLR,
  year={2021}
}

@inproceedings{zhou2024dino,
  title={Dino-wm: World models on pre-trained visual features enable zero-shot planning},
  author={Zhou, Gaoyue and Pan, Hengkai and LeCun, Yann and Pinto, Lerrel},
  booktitle=ICML,
  year=2025
}

@article{forrester1971counterintuitive,
  title={Counterintuitive behavior of social systems},
  author={Forrester, Jay W},
  journal={Theory and decision},
  volume={2},
  number={2},
  pages={109--140},
  year={1971},
  publisher={Springer}
}

@article{de2005cem,
  title={A tutorial on the cross-entropy method},
  author={De Boer, Pieter-Tjerk and Kroese, Dirk P and Mannor, Shie and Rubinstein, Reuven Y},
  journal={Annals of operations research},
  volume={134},
  number={1},
  pages={19--67},
  year={2005},
  publisher={Springer}
}

@inproceedings{chua2018cem,
  title={Deep reinforcement learning in a handful of trials using probabilistic dynamics models},
  author={Chua, Kurtland and Calandra, Roberto and McAllister, Rowan and Levine, Sergey},
  booktitle=NIPS,
  volume={31},
  year={2018}
}

@article{van2017neural,
  title={Neural discrete representation learning},
  author={Van Den Oord, Aaron and Vinyals, Oriol and others},
  journal={Advances in neural information processing systems},
  volume={30},
  year={2017}
}

@inproceedings{esser2021taming,
  title={Taming transformers for high-resolution image synthesis},
  author={Esser, Patrick and Rombach, Robin and Ommer, Bjorn},
  booktitle=CVPR,
  pages={12873--12883},
  year={2021}
}

@inproceedings{rombach2022high,
  title={High-resolution image synthesis with latent diffusion models},
  author={Rombach, Robin and Blattmann, Andreas and Lorenz, Dominik and Esser, Patrick and Ommer, Bj{\"o}rn},
  booktitle=NIPS,
  pages={10684--10695},
  year={2022}
}

@inproceedings{yu2021vector,
  title={Vector-quantized image modeling with improved vqgan},
  author={Yu, Jiahui and Li, Xin and Koh, Jing Yu and Zhang, Han and Pang, Ruoming and Qin, James and Ku, Alexander and Xu, Yuanzhong and Baldridge, Jason and Wu, Yonghui},
  booktitle=ICLR,
  year={2022}
}

@inproceedings{lee2022autoregressive,
  title={Autoregressive image generation using residual quantization},
  author={Lee, Doyup and Kim, Chiheon and Kim, Saehoon and Cho, Minsu and Han, Wook-Shin},
  booktitle=CVPR,
  pages={11523--11532},
  year={2022}
}

@inproceedings{yu2023language,
  title={Language Model Beats Diffusion--Tokenizer is Key to Visual Generation},
  author={Yu, Lijun and Lezama, Jos{\'e} and Gundavarapu, Nitesh B and Versari, Luca and Sohn, Kihyuk and Minnen, David and Cheng, Yong and Birodkar, Vighnesh and Gupta, Agrim and Gu, Xiuye and others},
  booktitle=ICLR,
  year=2024
}

@inproceedings{mentzer2023finite,
  title={Finite scalar quantization: Vq-vae made simple},
  author={Mentzer, Fabian and Minnen, David and Agustsson, Eirikur and Tschannen, Michael},
  booktitle=ICLR,
  year={2024}
}

@inproceedings{yu2024image,
  title={An image is worth 32 tokens for reconstruction and generation},
  author={Yu, Qihang and Weber, Mark and Deng, Xueqing and Shen, Xiaohui and Cremers, Daniel and Chen, Liang-Chieh},
  booktitle=NIPS,
  volume={37},
  pages={128940--128966},
  year={2024}
}

@inproceedings{bachmann2025flextok,
  title={FlexTok: Resampling Images into 1D Token Sequences of Flexible Length},
  author={Bachmann, Roman and Allardice, Jesse and Mizrahi, David and Fini, Enrico and Kar, O{\u{g}}uzhan Fatih and Amirloo, Elmira and El-Nouby, Alaaeldin and Zamir, Amir and Dehghan, Afshin},
  booktitle=ICML,
  year={2025}
}

@inproceedings{deng2009imagenet,
  title={Imagenet: A large-scale hierarchical image database},
  author={Deng, Jia and Dong, Wei and Socher, Richard and Li, Li-Jia and Li, Kai and Fei-Fei, Li},
  booktitle = CVPR,
  year={2009}
}

@article{sun2024autoregressive,
  title={Autoregressive Model Beats Diffusion: Llama for Scalable Image Generation},
  author={Sun, Peize and Jiang, Yi and Chen, Shoufa and Zhang, Shilong and Peng, Bingyue and Luo, Ping and Yuan, Zehuan},
  journal={arXiv preprint arXiv:2406.06525},
  year={2024}
}

@inproceedings{peebles2023scalable,
  title={Scalable diffusion models with transformers},
  author={Peebles, William and Xie, Saining},
  booktitle=ICCV,
  year={2023}
}

@article{li2024autoregressive,
  title={Autoregressive Image Generation without Vector Quantization},
  author={Li, Tianhong and Tian, Yonglong and Li, He and Deng, Mingyang and He, Kaiming},
  journal=NIPS,
  year={2024}
}

@inproceedings{weber2024maskbit,
  title={MaskBit: Embedding-free Image Generation via Bit Tokens},
  author={Weber, Mark and Yu, Lijun and Yu, Qihang and Deng, Xueqing and Shen, Xiaohui and Cremers, Daniel and Chen, Liang-Chieh},
  booktitle=TMLR,
  year={2024}
}

@inproceedings{chang2023muse,
  title={Muse: Text-to-image generation via masked generative transformers},
  author={Chang, Huiwen and Zhang, Han and Barber, Jarred and Maschinot, AJ and Lezama, Jos{\'e} and Jiang, Lu and Yang, Ming-Hsuan and Murphy, Kevin and Freeman, William T and Rubinstein, Michael and others},
  booktitle=ICML,
  year={2023}
}

@inproceedings{li2023mage,
  title={Mage: Masked generative encoder to unify representation learning and image synthesis},
  author={Li, Tianhong and Chang, Huiwen and Mishra, Shlok and Zhang, Han and Katabi, Dina and Krishnan, Dilip},
  booktitle=CVPR,
  year={2023}
}

@inproceedings{fan2024fluid,
  title={Fluid: Scaling Autoregressive Text-to-image Generative Models with Continuous Tokens},
  author={Fan, Lijie and Li, Tianhong and Qin, Siyang and Li, Yuanzhen and Sun, Chen and Rubinstein, Michael and Sun, Deqing and He, Kaiming and Tian, Yonglong},
  booktitle=ICLR,
  year={2025}
}

@inproceedings{johnson2016perceptual,
  title={Perceptual losses for real-time style transfer and super-resolution},
  author={Johnson, Justin and Alahi, Alexandre and Fei-Fei, Li},
  booktitle=ECCV,
  year={2016}
}

@inproceedings{esser2024scaling,
  title={Scaling rectified flow transformers for high-resolution image synthesis},
  author={Esser, Patrick and Kulal, Sumith and Blattmann, Andreas and Entezari, Rahim and M{\"u}ller, Jonas and Saini, Harry and Levi, Yam and Lorenz, Dominik and Sauer, Axel and Boesel, Frederic and others},
  booktitle=ICML,
  year={2024}
}

@article{gpt3,
  title={Language Models are Few-Shot Learners},
  author={Brown, Tom B. and Mann, Benjamin and Ryder, Nick and Subbiah, Melanie and Kaplan, Jared and Dhariwal, Prafulla and Neelakantan, Arvind and Shyam, Pranav and Sastry, Girish and Askell, Amanda and others},
  journal=NIPS,
  year={2020},
  url={https://proceedings.neurips.cc/paper/2020/file/1457c0d6bfcb4967418bfb8ac142f64a-Paper.pdf},
}

@inproceedings{chang2022maskgit,
  title={Maskgit: Masked generative image transformer},
  author={Chang, Huiwen and Zhang, Han and Jiang, Lu and Liu, Ce and Freeman, William T},
  booktitle=CVPR,
  year={2022}
}

@article{vaswani2017attention,
  title={Attention is all you need},
  author={Vaswani, Ashish and Shazeer, Noam and Parmar, Niki and Uszkoreit, Jakob and Jones, Llion and Gomez, Aidan N and Kaiser, {\L}ukasz and Polosukhin, Illia},
  journal=NIPS,
  year={2017}
}

@inproceedings{he2022masked,
  title={Masked autoencoders are scalable vision learners},
  author={He, Kaiming and Chen, Xinlei and Xie, Saining and Li, Yanghao and Doll{\'a}r, Piotr and Girshick, Ross},
  booktitle=CVPR,
  pages={16000--16009},
  year={2022}
}

@inproceedings{chen2020generative,
  title={Generative pretraining from pixels},
  author={Chen, Mark and Radford, Alec and Child, Rewon and Wu, Jeffrey and Jun, Heewoo and Luan, David and Sutskever, Ilya},
  booktitle=ICML,
  year={2020}
}

@article{salimans2016improved,
  title={Improved techniques for training gans},
  author={Salimans, Tim and Goodfellow, Ian and Zaremba, Wojciech and Cheung, Vicki and Radford, Alec and Chen, Xi},
  journal=NIPS,
  year={2016}
}

@article{miwa2025one,
  title={One-d-piece: Image tokenizer meets quality-controllable compression},
  author={Miwa, Keita and Sasaki, Kento and Arai, Hidehisa and Takahashi, Tsubasa and Yamaguchi, Yu},
  journal={arXiv preprint arXiv:2501.10064},
  year={2025}
}

@article{gao2023mdtv2,
  title={Mdtv2: Masked diffusion transformer is a strong image synthesizer},
  author={Gao, Shanghua and Zhou, Pan and Cheng, Ming-Ming and Yan, Shuicheng},
  journal={arXiv preprint arXiv:2303.14389},
  year={2023}
}

@inproceedings{tang2024hart,
  title={Hart: Efficient visual generation with hybrid autoregressive transformer},
  author={Tang, Haotian and Wu, Yecheng and Yang, Shang and Xie, Enze and Chen, Junsong and Chen, Junyu and Zhang, Zhuoyang and Cai, Han and Lu, Yao and Han, Song},
  booktitle=ICLR,
  year={2025}
}

@inproceedings{zheng2022movq,
  title={Movq: Modulating quantized vectors for high-fidelity image generation},
  author={Zheng, Chuanxia and Vuong, Tung-Long and Cai, Jianfei and Phung, Dinh},
  booktitle=NIPS,
  volume={35},
  pages={23412--23425},
  year={2022}
}

@inproceedings{hafner2019dream,
  title={Dream to control: Learning behaviors by latent imagination},
  author={Hafner, Danijar and Lillicrap, Timothy and Ba, Jimmy and Norouzi, Mohammad},
  booktitle=ICLR,
  year={2020}
}

@article{hafner2023mastering,
  title={Mastering diverse domains through world models},
  author={Hafner, Danijar and Pasukonis, Jurgis and Ba, Jimmy and Lillicrap, Timothy},
  journal={Nature},
  pages={1--7},
  year={2025},
  publisher={Nature Publishing Group}
}

@inproceedings{bar2025navigation,
  title={Navigation world models},
  author={Bar, Amir and Zhou, Gaoyue and Tran, Danny and Darrell, Trevor and LeCun, Yann},
  booktitle=CVPR,
  pages={15791--15801},
  year={2025}
}

@article{ha2018world,
  title={World models},
  author={Ha, David and Schmidhuber, J{\"u}rgen},
  journal={arXiv preprint arXiv:1803.10122},
  volume={2},
  number={3},
  year={2018}
}

@inproceedings{koh2021pathdreamer,
  title={Pathdreamer: A world model for indoor navigation},
  author={Koh, Jing Yu and Lee, Honglak and Yang, Yinfei and Baldridge, Jason and Anderson, Peter},
  booktitle=ICCV,
  pages={14738--14748},
  year={2021}
}

@inproceedings{
shah2021rapid,
title={Rapid Exploration for Open-World Navigation with Latent Goal Models},
author={Dhruv Shah and Benjamin Eysenbach and Nicholas Rhinehart and Sergey Levine},
booktitle={Annual Conference on Robot Learning (CoRL)},
year={2021},
url={https://openreview.net/forum?id=d_SWJhyKfVw}
}

@inproceedings{bruce2024genie,
  title={Genie: Generative interactive environments},
  author={Bruce, Jake and Dennis, Michael D and Edwards, Ashley and Parker-Holder, Jack and Shi, Yuge and Hughes, Edward and Lai, Matthew and Mavalankar, Aditi and Steigerwald, Richie and Apps, Chris and others},
  booktitle=ICML,
  year={2024}
}

@inproceedings{alonso2024diffusion,
  title={Diffusion for world modeling: Visual details matter in atari},
  author={Alonso, Eloi and Jelley, Adam and Micheli, Vincent and Kanervisto, Anssi and Storkey, Amos J and Pearce, Tim and Fleuret, Fran{\c{c}}ois},
  booktitle=NIPS,
  volume={37},
  pages={58757--58791},
  year={2024}
}

@inproceedings{valevski2024diffusion,
  title={Diffusion models are real-time game engines},
  author={Valevski, Dani and Leviathan, Yaniv and Arar, Moab and Fruchter, Shlomi},
  booktitle=ICLR,
  year={2025}
}

@inproceedings{yang2023learning,
  title={Learning interactive real-world simulators},
  author={Yang, Mengjiao and Du, Yilun and Ghasemipour, Kamyar and Tompson, Jonathan and Schuurmans, Dale and Abbeel, Pieter},
  booktitle=ICLR,
  volume={1},
  number={2},
  pages={6},
  year={2024}
}

@inproceedings{mendonca2023structured,
  title={Structured world models from human videos},
  author={Mendonca, Russell and Bahl, Shikhar and Pathak, Deepak},
  booktitle=RSS,
  year={2023}
}

@article{hu2023gaia,
  title={Gaia-1: A generative world model for autonomous driving},
  author={Hu, Anthony and Russell, Lloyd and Yeo, Hudson and Murez, Zak and Fedoseev, George and Kendall, Alex and Shotton, Jamie and Corrado, Gianluca},
  journal={arXiv preprint arXiv:2309.17080},
  year={2023}
}

@article{gao2024vista,
  title={Vista: A generalizable driving world model with high fidelity and versatile controllability},
  author={Gao, Shenyuan and Yang, Jiazhi and Chen, Li and Chitta, Kashyap and Qiu, Yihang and Geiger, Andreas and Zhang, Jun and Li, Hongyang},
  journal={Advances in Neural Information Processing Systems},
  volume={37},
  pages={91560--91596},
  year={2024}
}

@inproceedings{nie2025wmnav,
  title={Wmnav: Integrating vision-language models into world models for object goal navigation},
  author={Nie, Dujun and Guo, Xianda and Duan, Yiqun and Zhang, Ruijun and Chen, Long},
  booktitle=IROS,
  year={2025}
}

@inproceedings{yao2025navmorph,
  title={NavMorph: A Self-Evolving World Model for Vision-and-Language Navigation in Continuous Environments},
  author={Yao, Xuan and Gao, Junyu and Xu, Changsheng},
  booktitle=ICCV,
  year={2025}
}

@inproceedings{zhao2025drivedreamer,
  title={Drivedreamer-2: Llm-enhanced world models for diverse driving video generation},
  author={Zhao, Guosheng and Wang, Xiaofeng and Zhu, Zheng and Chen, Xinze and Huang, Guan and Bao, Xiaoyi and Wang, Xingang},
  booktitle=AAAI,
  volume={39},
  number={10},
  pages={10412--10420},
  year={2025}
}

@inproceedings{williams2016aggressive,
  title={Aggressive driving with model predictive path integral control},
  author={Williams, Grady and Drews, Paul and Goldfain, Brian and Rehg, James M and Theodorou, Evangelos A},
  booktitle={2016 IEEE international conference on robotics and automation (ICRA)},
  pages={1433--1440},
  year={2016},
  organization={IEEE}
}

@article{simeoni2025dinov3,
  title={Dinov3},
  author={Sim{\'e}oni, Oriane and Vo, Huy V and Seitzer, Maximilian and Baldassarre, Federico and Oquab, Maxime and Jose, Cijo and Khalidov, Vasil and Szafraniec, Marc and Yi, Seungeun and Ramamonjisoa, Micha{\"e}l and others},
  journal={arXiv preprint arXiv:2508.10104},
  year={2025}
}

@inproceedings{du2023learning,
  title={Learning universal policies via text-guided video generation},
  author={Du, Yilun and Yang, Sherry and Dai, Bo and Dai, Hanjun and Nachum, Ofir and Tenenbaum, Josh and Schuurmans, Dale and Abbeel, Pieter},
  booktitle=NIPS,
  volume={36},
  pages={9156--9172},
  year={2023}
}

@inproceedings{parmar2022aliased,
  title={On aliased resizing and surprising subtleties in gan evaluation},
  author={Parmar, Gaurav and Zhang, Richard and Zhu, Jun-Yan},
  booktitle=CVPR,
  pages={11410--11420},
  year={2022}
}

@article{dasari2019robonet,
  title={Robonet: Large-scale multi-robot learning},
  author={Dasari, Sudeep and Ebert, Frederik and Tian, Stephen and Nair, Suraj and Bucher, Bernadette and Schmeckpeper, Karl and Singh, Siddharth and Levine, Sergey and Finn, Chelsea},
  journal={Annual Conference on Robot Learning (CoRL)},
  year={2019}
}

@inproceedings{micheli2024efficient,
  title={Efficient world models with context-aware tokenization},
  author={Micheli, Vincent and Alonso, Eloi and Fleuret, Fran{\c{c}}ois},
  booktitle=ICML,
  year={2024}
}

@inproceedings{scannell2025discrete,
  title={Discrete codebook world models for continuous control},
  author={Scannell, Aidan and Nakhaei, Mohammadreza and Kujanp{\"a}{\"a}, Kalle and Zhao, Yi and Luck, Kevin Sebastian and Solin, Arno and Pajarinen, Joni},
  booktitle=ICLR,
  year={2025}
}
